\documentclass[journal]{IEEEtran}

\usepackage{silence}
\usepackage{bbm}

\usepackage[T1]{fontenc}
\usepackage{amsmath,amssymb,amsfonts,mathrsfs,bm}
\usepackage{mathtools}
\usepackage{amsthm}
\usepackage{cite}
\usepackage[shortlabels]{enumitem}
\usepackage{graphicx}
\usepackage{epstopdf}
\usepackage{url}
\usepackage{colortbl}
\usepackage{multirow}
\usepackage[table,dvipsnames]{xcolor}
\usepackage[normalem]{ulem}
\usepackage{array}
\newcolumntype{L}[1]{>{\raggedright\let\newline\\\arraybackslash\hspace{0pt}}m{#1}}
\newcolumntype{C}[1]{>{\centering\let\newline\\\arraybackslash\hspace{0pt}}m{#1}}
\newcolumntype{R}[1]{>{\raggedleft\let\newline\\\arraybackslash\hspace{0pt}}m{#1}}

\makeatletter
\let\MYcaption\@makecaption
\makeatother
\usepackage[font=footnotesize]{subcaption}
\makeatletter
\let\@makecaption\MYcaption
\makeatother

\usepackage{xparse}



\usepackage{glossaries}
\newacronym{wrt}{w.r.t.}{with respect to}
\newacronym{RHS}{R.H.S.}{right-hand side}
\newacronym{LHS}{L.H.S.}{left-hand side}
\newacronym{iid}{i.i.d.}{independent and identically distributed}

\usepackage{float}

\ifx\notloadhyperref\undefined
	\ifx\loadbibentry\undefined
		\usepackage[hidelinks,hypertexnames=false]{hyperref} 
	\else
		\usepackage{bibentry}
		\makeatletter\let\saved@bibitem\@bibitem\makeatother
		\usepackage[hidelinks,hypertexnames=false]{hyperref}
		\makeatletter\let\@bibitem\saved@bibitem\makeatother
	\fi
\else
	\relax
\fi

\usepackage[capitalize]{cleveref}
\crefname{equation}{}{}
\Crefname{equation}{}{}
\crefname{claim}{claim}{claims}
\crefname{step}{step}{steps}
\crefname{line}{line}{lines}
\crefname{condition}{condition}{conditions}
\crefname{dmath}{}{}
\crefname{dseries}{}{}
\crefname{dgroup}{}{}
\crefname{Theorem}{Theorem}{Theorems}
\crefname{Corollary}{Corollary}{Corollaries}
\crefname{Proposition}{Proposition}{Propositions}
\crefname{Lemma}{Lemma}{Lemmas}
\crefname{Definition}{Definition}{Definitions}
\crefname{Example}{Example}{Examples}
\crefname{Assumption}{Assumption}{Assumptions}
\crefname{Remark}{Remark}{Remarks}
\crefname{Rem}{Remark}{Remarks}
\crefname{remarks}{Remarks}{Remarks}
\crefname{Exercise}{Exercise}{Exercises}
\crefname{Theorem_A}{Theorem}{Theorems}
\crefname{Corollary_A}{Corollary}{Corollaries}
\crefname{Proposition_A}{Proposition}{Propositions}
\crefname{Lemma_A}{Lemma}{Lemmas}
\crefname{Definition_A}{Definition}{Definitions}

\usepackage{algorithm,algorithmic}

\ifx\loadbreqn\undefined
	\relax
\else
	\usepackage{breqn} 
\fi

\ifx\renewtheorem\undefined
\ifx\useTheoremCounter\undefined
\newtheorem{Theorem}{Theorem}
\newtheorem{Corollary}{Corollary}
\newtheorem{Proposition}{Proposition}

\else
\newtheorem{Theorem}{Theorem}

\fi

\newtheorem{Assumption}{Assumption}

\fi

\theoremstyle{remark}

\theoremstyle{plain}

\newcommand{\Real}{\mathbb{R}}

\newcommand{\calN}{\mathcal{N}}

\newcommand{\calU}{\mathcal{U}}

\newcommand{\bA}{\mathbf{A}}

\newcommand{\bI}{\mathbf{I}}

\newcommand{\bp}{\mathbf{p}}

\newcommand{\bq}{\mathbf{q}}

\newcommand{\bV}{\mathbf{V}}
\newcommand{\bw}{\mathbf{w}}

\newcommand{\bx}{\mathbf{x}}

\newcommand{\by}{\mathbf{y}}

\newcommand{\bbD}{\mathbb{D}}

\newcommand{\bbU}{\mathbb{U}}

\DeclareSymbolFont{bsfletters}{OT1}{cmss}{bx}{n}
\DeclareSymbolFont{ssfletters}{OT1}{cmss}{m}{n}
\DeclareMathSymbol{\bsfGamma}{0}{bsfletters}{'000}
\DeclareMathSymbol{\ssfGamma}{0}{ssfletters}{'000}
\DeclareMathSymbol{\bsfDelta}{0}{bsfletters}{'001}
\DeclareMathSymbol{\ssfDelta}{0}{ssfletters}{'001}
\DeclareMathSymbol{\bsfTheta}{0}{bsfletters}{'002}
\DeclareMathSymbol{\ssfTheta}{0}{ssfletters}{'002}
\DeclareMathSymbol{\bsfLambda}{0}{bsfletters}{'003}
\DeclareMathSymbol{\ssfLambda}{0}{ssfletters}{'003}
\DeclareMathSymbol{\bsfXi}{0}{bsfletters}{'004}
\DeclareMathSymbol{\ssfXi}{0}{ssfletters}{'004}
\DeclareMathSymbol{\bsfPi}{0}{bsfletters}{'005}
\DeclareMathSymbol{\ssfPi}{0}{ssfletters}{'005}
\DeclareMathSymbol{\bsfSigma}{0}{bsfletters}{'006}
\DeclareMathSymbol{\ssfSigma}{0}{ssfletters}{'006}
\DeclareMathSymbol{\bsfUpsilon}{0}{bsfletters}{'007}
\DeclareMathSymbol{\ssfUpsilon}{0}{ssfletters}{'007}
\DeclareMathSymbol{\bsfPhi}{0}{bsfletters}{'010}
\DeclareMathSymbol{\ssfPhi}{0}{ssfletters}{'010}
\DeclareMathSymbol{\bsfPsi}{0}{bsfletters}{'011}
\DeclareMathSymbol{\ssfPsi}{0}{ssfletters}{'011}
\DeclareMathSymbol{\bsfOmega}{0}{bsfletters}{'012}
\DeclareMathSymbol{\ssfOmega}{0}{ssfletters}{'012}

\newcommand{\btheta}{\bm{\theta}}

\DeclareMathOperator*{\argmax}{arg\,max}
\DeclareMathOperator*{\argmin}{arg\,min}

\DeclareMathOperator{\tr}{tr}

\DeclarePairedDelimiter\parens{(}{)}

\DeclarePairedDelimiter\braces{\{}{\}}

\newcommand{\qednew}{\nobreak \ifvmode \relax \else
      \ifdim\lastskip<1.5em \hskip-\lastskip
      \hskip1.5em plus0em minus0.5em \fi \nobreak
      \vrule height0.75em width0.5em depth0.25em\fi}

\newcommand{\nn}{\nonumber\\}

\newcommand{\T}{^{\intercal}}
\newcommand{\ud}{\mathrm{d}}

\newcommand{\indicator}[1]{{\bf 1}_{\braces*{#1}}}
\newcommand{\indicatore}[1]{{\bf 1}_{#1}}

\newcommand{\ofrac}[1]{{\frac{1}{#1}}}

\newcommand{\ceil}[1]{\left\lceil{#1}\right\rceil}

\newcommand{\ip}[2]{{\left\langle{#1},\, {#2}\right\rangle}}

\newcommand{\cond}[2]{\left. {#1}\, \middle| \, {#2} \right.}

\DeclareDocumentCommand \ifcond {m m} {%
	{#1} %
	\IfValueT{#2}{\, \middle|\, {#2}}%
}

\DeclareDocumentCommand \P { g d() g } {%
	\IfNoValueTF {#3} 
	{%
		\IfNoValueTF {#1} 
		{%
			\IfNoValueTF {#2}
			{%
				\mathbb{P}%
			}%
			{%
				\mathbb{P}\left({#2}\right)%
			}%
		}%
		{%
			\IfNoValueTF {#2}
			{%
				\mathbb{P}_{#1}%
			}%
			{%
				\mathbb{P}_{#1}\left({#2}\right)%
			}%
		}%
	}%
	{%
		\IfNoValueTF {#1} 
		{%
			\mathbb{P}\left(\cond{#2}{#3}\right)%
		}%
		{%
			\mathbb{P}_{#1}\left(\cond{#2}{#3}\right)%
		}%
	}%
}

\DeclareDocumentCommand \E { g o g } {%
	\IfNoValueTF {#3} 
	{%
		\IfNoValueTF {#1} 
		{%
			\IfNoValueTF {#2}
			{%
				\mathbb{E}%
			}%
			{%
				\mathbb{E}\left[{#2}\right]%
			}%
		}%
		{%
			\IfNoValueTF {#2}
			{%
				\mathbb{E}_{#1}%
			}%
			{%
				\mathbb{E}_{#1}\left[{#2}\right]%
			}%
		}%
	}%
	{%
		\IfNoValueTF {#1} 
		{%
			\mathbb{E}\left[\cond{#2}{#3}\right]%
		}%
		{%
			\mathbb{E}_{#1}\left[\cond{#2}{#3}\right]%
		}%
	}%
}

\definecolor{gray90}{gray}{0.9}

\ifx\nohighlights\undefined

	\newcommand{\msout}[1]{\text{\color{green} \sout{\ensuremath{#1}}}}
	\newcommand{\del}[1]{{\color{green}\ifmmode \msout{#1}\else\sout{#1}\fi}}
\else

	\newcommand{\msout}[1]{#1}
	\newcommand{\del}[1]{#1}
\fi

\newcommand{\hide}[1]{}

\renewcommand{\figurename}{Fig.}

\graphicspath{{./Figures/}} 
\ifx\diagnoselabel\undefined
	\relax
\else
	\makeatletter
	 \def\@testdef #1#2#3{%
		 \def\reserved@a{#3}\expandafter \ifx \csname #1@#2\endcsname
		\reserved@a  \else
	 \typeout{^^Jlabel #2 changed:^^J%
	 \meaning\reserved@a^^J%
	 \expandafter\meaning\csname #1@#2\endcsname^^J}%
	 \@tempswatrue \fi}
	\makeatother
\fi

\usepackage{microtype}

\begin{document}

\title{Learning Orthogonal Projections in Linear Bandits}
\author{Qiyu~Kang, and Wee~Peng~Tay, \IEEEmembership{Senior~Member,~IEEE}

\thanks{This research is supported in part by the Singapore Ministry of Education Academic Research Fund Tier 2 grant MOE2018-T2-2-019.}
\thanks{The authors are with the School of Electrical and Electronic Engineering, Nanyang Technological University, Singapore. Email: kang0080@e.ntu.edu.sg, wptay@ntu.edu.sg.}  }

\maketitle

\begin{abstract}
In a linear stochastic bandit model, each arm is a vector in an Euclidean space and the observed return at each time step is an unknown linear function of the chosen arm at that time step. In this paper, we investigate the problem of learning the best arm in a linear stochastic bandit model, where each arm's expected reward is an unknown linear function of the projection of the arm onto a subspace. We call this the projection reward. Unlike the classical linear bandit problem in which the observed return corresponds to the reward, the projection reward at each time step is unobservable. Such a model is useful in recommendation applications where the observed return includes corruption by each individual's biases, which we wish to exclude in the learned model. In the case where there are finitely many arms, we develop a strategy to achieve $O(|\bbD|\log n)$ regret, where $n$ is the number of time steps and $|\bbD|$ is the number of arms. In the case where each arm is chosen from an infinite compact set, our strategy achieves $O(n^{2/3}(\log{n})^{1/2})$ regret. Experiments verify the efficiency of our strategy.
\end{abstract}


\IEEEpeerreviewmaketitle

\section{Introduction}\label{sec:intro}

\IEEEPARstart{M}{ulti-armed} bandit (MAB) problems, introduced by Robbins \cite{robbins1952some} model the exploration and exploitation trade-off of sequential decision making under uncertainty. In its most basic paradigm, at each time step, the decision maker is given $d$ decisions or arms from which he is supposed to select one, and as a response, he observes a stochastic reward after each decision-making. The arm chosen at each time step is based on the information gathered at all the previous time steps. Therefore, in order to maximize the expected cumulative reward, exploitation of the current empirically best arm and exploration of less frequently chosen arms should be balanced carefully.  Stochastic independence of rewards is assumed for different arms in some works like \cite{auer2002finite, lai1985asymptotically, audibert2009exploration, agrawal1995sample, burnetas1996optimal,kangrecomm, Liu2017, KlosneeMueller2018, KanTay:J19}. In other works like \cite{pandey2007multi,presman1990sequential,goldenshluger2009woodroofe,abbasi2011improved,rusmevichientong2010linearly}, reward dependence between arms is assumed, which allows the decision maker to gather information for more than one arm at each time step. One such specific assumption is that arms are vectors containing numeric elements, and the expected reward of choosing each arm is an inner product of the arm vector with an unknown parameter vector \cite{Gai2012, agrawal2013thompson, abbasi2011improved,rusmevichientong2010linearly}. In \cite{abbasi2011improved, rusmevichientong2010linearly}, the authors proposed effective strategies that balance exploration and exploitation based on the \emph{optimism-in-the-face-of-uncertainty} principle. This principle maintains a high probability confidence set for the estimated parameter vector, and at each time step, the decision maker chooses an arm and a parameter vector from the decision set and the confidence set respectively, so that their inner product is maximized. Note that the independent $d$-armed bandit problem can be viewed as a special case of the general linear stochastic bandit problem, where the set of $d$ available decisions serves as a standard orthonormal basis of $\Real^d$\cite{abbasi2011improved}.  

The linear stochastic bandit model has been successfully applied to some real-world problems like personalized news article recommendation \cite{li2010contextual}, advertisement selection \cite{abe2003reinforcement}, and information retrieval \cite{yue2011linear}.  For example, in news recommendation applications, typical features, including the news' topic categories, the users' race, gender, location, etc., can be treated as components in each arm or decision vector, while the users' clicks are the rewards. A plausible recommendation scheme is to get as many clicks as possible. Another application is in sequential clinical trials \cite{villar2015multi,dani2008stochastic,kuleshov2014algorithms} whose aims are to balance the correct identification of the best treatment (exploration) and the effectiveness of the treatment during the trials (exploitation). The classical sequential clinical trials containing $d$ different drugs can be modeled as an independent $d$-armed bandit problem. However, this may be impractical as treatments may utilize multiple drugs simultaneously. A more feasible and general way is to model the treatment decision as a set of mixed drugs instead. An arm corresponds to a mixed drugs treatment with specific dosages of each kind. The reward at each time step is the curative efficacy after applying the mixed drug treatment to a patient at each trial. 

However, in some scenarios, the decision maker is more interested in some criterion other than maximizing the cumulative reward in the standard linear stochastic bandit model. One example is a discrimination-aware movie recommendation system. To avoid racially discriminatory recommendations, a user's race should not play any role in the recommendation system. However, the observed reward (number of clicks) may be biased by racial factors. A black user may have a history of following a particular black actor, but it does not mean that he or she should always be recommended movies with black actors.  
For example, Netflix last year angered black subscribers with targeted posters containing black actors no matter how minor their roles in the film are \cite{Black}. In principle, to avoid discrimination, protected or sensitive attributes such as race or gender should not be included in recommendation algorithms. Nonetheless, discarding such attributes directly and modeling the reward as a linear function of the other unprotected attributes may introduce system bias during the learning process. In another example, some clinical trials seek to maximize the curative effect on one disease of a mixed drugs treatment, while at the same time the patients who have this disease may concurrently have another disease. These patients need to take other drugs that have a positive or negative impact on the targeted disease. One typical example is the treatment of hypertension in chronic kidney disease. Hypertension is present in more than $80\%$ of the patients with chronic kidney disease, and drugs like ACEi and ARB targeted for reducing proteinuria may also have an effect on blood pressure\cite{toto2005treatment}. To study the effect of a mixed drug treatment targeted for controlling hypertension, the decision maker is supposed to discount the impact of drugs like ACEi and ARB in the decision process.  

In this paper, we propose a linear stochastic bandit formulation that maximizes the (expected) cumulative reward over a subspace of decision attributes, based on the reward observed for the full space. Specifically, each arm is projected orthogonally onto a target subspace $\bbU$. The reward is then decomposed into two components, one of which is due to $\bbU$, and the other is due to $\bbU^\perp$. We call the first component the \emph{projection reward} and the second component the \emph{corruption}. We develop a strategy that achieves\footnote{For non-negative functions $f(n)$ and $g(n)$, we write $f(n)=O(g(n))$ if $\limsup_{n\to\infty} f(n)/g(n)<\infty$, $f(n) = o(g(n))$ if $\lim_{n\to\infty} f(n)/g(n)=0$, and $f(n) = \Theta(g(n))$ if $0 < \liminf_{n\to\infty} f(n)/g(n) \leq \limsup_{n\to\infty} f(n)/g(n)<\infty$.} $O(|\bbD|\log n)$ cumulative projection regret when there are finitely many $|\bbD|$ arms and where $n$ is the number of time steps. In the case where the arms are drawn from an infinite compact set, we achieve a cumulative projection regret of $O(n^{2/3}(\log{n})^{1/2})$. Here, the projection regret at each time step is defined as the difference between the current expected projection reward of making one decision and the oracle best expected projection reward. This algorithm is based on the $\epsilon_t$-greedy policy \cite{auer2002finite}, which is a simple and well-known algorithm for the standard finite multi-armed bandit problem. 

\subsection{Related Work}

Two categories of work are related to our problem of interest: 1)  the linear stochastic bandit model; and 2) the multi-objective Pareto bandit problem.

In the linear stochastic bandit model, the decision maker predicts the reward of an arm in $\Real^d$ based on the given context vector of this decision. In \cite{AuerConfidence}, the author proposed an algorithm based on least squares estimation and high probability confidence bounds, and showed that it has $O(\sqrt{dn}\log^{3/2}(n|\bbD|))$ regret upper bound for the case of finitely many decisions, where $|\bbD|$ is the number of arms. The work \cite{dani2008stochastic} extended \cite{AuerConfidence}  to the problem with an arbitrary compact set of arms and presented a policy with $O(d\sqrt{n}\log^{3/2}n)$ regret upper bound. References \cite{abbasi2011improved,rusmevichientong2010linearly} further improved the regret upper bound using smaller confidence sets established using martingale techniques. In \cite{wang2016learning}, the authors proposed a linear bandit formulation with hidden features where the reward of choosing one decision is the sum of two components, one of which is a linear function of the observable features, and the other is a linear function of the unobservable features. They applied a upper confidence bound (UCB)-type linear bandit policy with a coordinate descent \cite{uschmajew2012local,friedman2010regularization} algorithm in which estimating the hidden features and the unknown coefficients jointly over time is achieved. Our work is different from the above-mentioned papers in that we seek to maximize only the cumulative projection reward rather than the reward observed for the full space. Furthermore, the projection reward is formed by projecting both the decision arm and context vectors into a subspace.

In the stochastic multi-objective multi-armed bandit (MOMAB) problem \cite{drugan2014pareto,drugan2013designing,yahyaa2014annealing,yahyaa2014scalarized}, the reward of making one decision is a vector rather than the scalar reward in the standard multi-armed bandit problem. Because of the possible conflicting objectives, a set of Pareto optimal arms \cite{drugan2013designing} \cite{zitzler2002performance} is considered in the MOMAB problem instead of a single best arm. Scalarization techniques can be used to transform a MOMAB problem to a single objective MAB, and an arm belonging to the Pareto optimal set is regarded as a best arm in a particular scalarization function. In our work, we consider the decomposition of the reward where the decomposition is similar to a two-objective MAB applied with a linear scalarization function \cite{drugan2013designing}. However, the difference between our model and MOMAB with linear scalarization is that we can only observe the sum of the (desired) projection reward and (undesired) corruption at each time step rather than a reward vector consisting of the two components as in MOMAB. Furthermore, our objective is to minimize the cumulative projection regret rather than minimizing the three types of regret defined in \cite{drugan2013designing}. 

\subsection{Our Contributions}
In this paper, we study the orthogonal projection problem in linear bandits, where our aim is to maximize the expected projection reward, which is an unknown linear function of a projection of the chosen arm onto a target subspace $\bbU \subseteq \Real^d$. In the case where there are finitely many arms, we develop a strategy to achieve $O(|\bbD|\log n)$ cumulative projection regret, where $n$ is the number of time steps and $|\bbD|$ is the number of arms. In the case where the arm is chosen from an infinite compact set, our strategy achieves $O(n^{2/3}(\log{n})^{1/2})$ cumulative projection regret. 

In the linear stochastic bandit literature, the best cumulative regret upper bound for the case of a compact decision set is $O(d\sqrt{n}\, \text{polylog}(n))$ \cite{abbasi2011improved, rusmevichientong2010linearly}, where $\text{polylog}(n)$ means a polynomial in $\log n$.  If a further smoothness assumption same as that in \cite{rusmevichientong2010linearly} is made, we show that it is possible to achieve the $O(d\sqrt{n})$ projection regret upper bound in our problem formulation. However, the existence of a policy with $O(d\sqrt{n}\, \text{polylog}(n))$ cumulative projection regret for the general compact decision set remains an open question for our problem. We verify our proposed policies on both synthetic simulations and experiments on a wine quality dataset \cite{cortez2009modeling}.

The rest of this paper is organized as follows. In \cref{sec:model}, we present our system model and assumptions. In \cref{sec:strategy}, we introduce our strategies and prove that they achieve sublinear cumulative projection regret. In \cref{sec:simulation}, we present simulations and experiments to verify the performance of our strategies. \Cref{sec:conclusion} concludes the paper. 

\emph{Notations:}  
We use $E^c$ to denote the complement of the event $E$. The indicator function $\indicatore{A}(\omega)=1$ if and only if $\omega\in A$. $\bI$ is the identity matrix. For $\bx, \by\in\Real^d$, let $\bx\T$ be the transpose of $\bx$ , and $\ip{\bx}{\by}=\bx\T \by$. We use $\Vert \bx\Vert_2$ to denote the $L_2$-norm $\sqrt{\ip{\bx}{\bx}}$ and $\Vert \bx\Vert_\bA$ to denote the weighted $L_2$-norm $\sqrt{\ip{\bx}{\bA\bx}}$ for any $\bA\in \Real^{d\times d}$ is a positive definite matrix. We use  $\calU(a,b)$ to denote the continuous uniform distribution whose density function has support $[a,b]$, and $\calN(0,\vartheta^2)$ to denote the Gaussian distribution with variance $\vartheta^2$ and mean $0$. 


\section{System Model}
\label{sec:model}
Let $\bbD\subseteq \Real^d$ be a compact set of decisions or arms from which the decision maker has to choose an arm $X_t$ at each time step $t$. The observed return $r_t$ after choosing $X_t$ is given by
\begin{align}
  r_{t} = \ip{X_t}{\btheta} + \eta_t,
  \label{eq:linear_reward}
\end{align}
where $\btheta\in\Real^d$ is a fixed but unknown parameter and $\eta_t$ is a zero mean \gls{iid} random noise. 

In the standard linear bandit problem, the performance of a policy is measured by the difference between the decision maker's cumulative reward and the cumulative reward achieved by the oracle policy with knowledge of $\btheta$. Formally, the goal of the decision maker is to minimize the \emph{cumulative regret} over $n$ time steps defined by 
\begin{align}
  R(n) = \sum_{t=1}^{n} \big(\ip{X^*}{\btheta} - \E[\ip{X_t}{\btheta}]\big),
  \label{eq:cumregret}
\end{align}
where $X^* = \argmax_{\bx\in \bbD}\ip{\bx}{\btheta}$ is the optimal arm in the standard linear bandit problem.

In our orthogonal projection linear bandit model, we consider a decomposition of the observed return into two components. Let $\bbU$, whose dimension is not more than $d$, be a subspace of the Euclidean vector space $\Real^d$. Let $\bbU^{\perp}=\{\bq\in \Real^d: \ip{\bq}{\bp} = 0, \forall \bp \in \bbU\}$ be the orthogonal complement of $\bbU$. It is a standard result \cite{linearalgebra} that $\Real^d = \bbU \oplus \bbU^{\perp}$, and each $\bw\in \Real^d$ can be written uniquely as a sum $\bp+\bq$, where $\bp\in \bbU$ and $\bq\in \bbU^{\perp}$. We define the linear orthogonal projection operator as $P_\bbU:\Real^d \mapsto \bbU$ such that $P_\bbU(\bw)= \bp$. The operator $P_\bbU$ can be represented by a $d\times d$ matrix whose rank is the dimension of $\bbU$. The expected observed return can therefore be decomposed as
\begin{align}
  \ip{X_t}{\btheta} &= \ip{P_\bbU(X_t) + P_{\bbU^{\perp}}(X_t)}{P_\bbU(\btheta) + P_{\bbU^{\perp}}(\btheta)}\nonumber \\ 
  &= \ip{P_\bbU(X_t)}{P_{\bbU}(\btheta)} + \ip{P_{\bbU^{\perp}}(X_t)}{P_{\bbU^{\perp}}(\btheta)}.
  \label{eq:reward_decomp}
\end{align}
We call $\ip{P_\bbU(X_t)}{P_\bbU(\theta)}$ the \emph{projection reward} on the subspace $\bbU$, which is denoted by $r_{\bbU,t}$. The other part of the observed return, $\ip{P_{\bbU^{\perp}}(X_t)}{P_{\bbU^{\perp}}(\btheta)}$, is called the \emph{corruption}. Different from the standard linear bandit model, where the reward can be observed (with random perturbation), in our orthogonal projection linear bandit model, we do not directly observe the projection reward $r_{\bbU,t}$. Furthermore,  the \emph{cumulative projection regret} is defined as  
\begin{align}
  R_\bbU(n) = \sum_{t=1}^{n}\big( \ip{P_{\bbU}(X_{\bbU}^*(\btheta))}{P_\bbU(\btheta)} - \E[\ip{P_\bbU(X_t)}{P_\bbU(\btheta)}]\big),
  \label{eq:cumprojregret}
\end{align}
where 
\begin{align}\label{eq:XU*}
  X_{\bbU}^*(\btheta)= \argmax_{\bx\in \bbD}\ip{P_\bbU(\bx)}{P_\bbU(\btheta)}
\end{align} 
is the \emph{best projection arm}. The objective in our model is to minimize the cumulative projection regret $R_\bbU(n)$ rather than the cumulative regret $R(n)$, base on the observed returns $r_1,\ldots,r_{n-1}$.

For a concrete illustration, consider again the discrimination prevention problem in a movie recommendation system described in \cref{sec:intro}. An arm is a movie to be recommended to a specific set of users and is represented by the features of the movie. The features include the movie genre, movie length, music types used, ages and races of the actors or actresses, etc. Suppose that each arm is a $d$ dimensional vector, where some of its elements like the ages and races may lead to discriminatory recommendations and are ``protected''. To eliminate the effect of those features when performing the recommendations, the system should consider only the non-protected dimensions. In \cite{calders2013controlling}, the authors proposed to control the discrimination effect in a linear regression model, where $P_\bbU$ in their formulation can also be viewed as a projection matrix.

We use $\text{span}(\bbD)$ to denote the set of finite linear combinations of arms in $\bbD$. In the ideal case, if at each time step $t$ the return $r_t$ is available without the random noise $\eta_t$, and if $\bbD$ is specified by linear inequality constraints, the problem then degenerates to a linear programming problem that finds $
\bx\in \bbD$ to maximize
\begin{align}
  & \ip{P_\bbU(\bx)}{P_\bbU(\btheta)} \nonumber \\
  & = \ip{P_\bbU(\bx)}{\btheta} \nonumber \\
  & = \ip{P_\bbU(\bx)}{P_{\text{span}(\bbD)}(\btheta)+ P_{\text{span}(\bbD)^\perp}(\btheta)}\nonumber \\
  & = \ip{P_\bbU(\bx)}{P_{\text{span}(\bbD)}(\btheta)}+ \ip{P_\bbU(\bx)}{P_{\text{span}(\bbD)^\perp}(\btheta)}.
  \label{eq:linear_degenerate}
\end{align}
However, since $P_{\text{span}(\bbD)^\perp}(\btheta)$ is unobservable even in the ideal case, and $\argmax_{\bx\in \bbD}$ $\ip{P_\bbU(\bx)}{P_{\text{span}(\bbD)}(\btheta)}$ is not equal to $X_{\bbU}^*(\btheta)$ in \cref{eq:XU*} when $\ip{P_\bbU(\bx)}{P_{\text{span}(\bbD)^\perp}(\btheta)}\ne 0$, a linear cumulative projection regret is therefore inevitable in the worst case. In order to get a sublinear cumulative regret, further assumptions are required. 

A possible assumption is $\ip{P_\bbU(\bx)}{P_{\text{span}(\bbD)^\perp}(\btheta)}= 0$ for all $\bx$ and $\btheta$. This is equivalent to saying $\bbU\subseteq \text{span}(\bbD)$. To see this, note that this assumption means that $P_{\text{span}(\bbD)^\perp}(\btheta)\in \text{kernel}(P_\bbU) = P_\bbU^{\perp} = P_{\bbU^\perp}$. It follows that $\text{span}(\bbD)^\perp \subseteq \bbU^{\perp}$, which is equivalent to $\bbU \subseteq \text{span}(\bbD)$. We make the following assumption throughout this paper.

\begin{Assumption}\label{as:as1}
   We have $\bbU\subseteq \text{span}(\bbD)$, the dimension of $\text{span}(\bbD)$ is $k$, where $\dim(\bbU) \le k\le d$, and $\|P_{\bbU}(\btheta)\|_2>0$. Furthermore, the decision maker has access to $k$ linearly independent arms in  $\bbD$.
\end{Assumption}

We use $\bbD_k \subseteq \bbD$ to denote the set containing the $k$ linearly independent arms in  \cref{as:as1}. The condition $\|P_{\bbU}(\btheta)\|_2>0$ guarantees $\btheta\notin \bbU^{\perp}$.  We also make another assumption which is standard in the literature for the linear bandit problem.

\begin{Assumption}\label{as:as2}
 We have $\Vert \btheta \Vert_2\le S$, and $\max_{\bx\in \bbD} \Vert \bx\Vert_2 \le Z$, where $S$ and $Z$ are positive finite constants. The random variables $\eta_t$ are drawn \gls{iid} from a zero-mean sub-Gaussian distribution with parameter $\vartheta$, i.e., $\E[e^{\xi\eta_t}]\le e^{\xi^2\vartheta^2/2}$ for all $\xi\in \Real$.
\end{Assumption}

\section{Strategies and Regret Analysis}\label{sec:strategy}
\subsection{Greedy Projection Strategy}\label{ssec:GENTRY}

The decision set $\bbD$ may contain finitely or infinitely many arms. In this section, we present a strategy with two settings, the first is applicable for the finite-arm case, and the other is applicable for the infinite-arm case. In the following, we use $\hat{\btheta}_t$ to denote the $L_2$ regularized least square estimate of $\btheta$ with parameter $\lambda> 0$, after the decision making at time step $t$:
\begin{align}
  \hat{\btheta}_t =\bV_t^{-1}\sum_{i = 1}^{t} r_iX_i,
    \label{eq:estimate}
\end{align}
where $\bV_t=\lambda \bI + \sum_{i=1}^{t} X_iX_i\T$. We propose a Greedy projEctioN sTRategY (GENTRY) that accounts for the corruption in the observed return to achieve sublinear cumulative projection regret. GENTRY is based on the $\epsilon_t$-greedy policy \cite{auer2002finite}, which is a simple and well-known algorithm for the standard finite multi-armed bandit problem. At each time step $t$, the $\epsilon_t$-greedy policy chooses with probability $1-\epsilon_t$ the arm with the highest empirical average reward, and with probability $\epsilon_t$ a random arm. Since in our problem, we focus our attention on the projection reward, GENTRY chooses with probability $1-\epsilon_t$ the arm with the highest empirical average projection reward defined as:
\begin{align}
  \bar{r}_{t-1}(\bx) = \ip{P_\bbU(\bx)}{P_\bbU(\hat{\btheta}_{t-1})},
  \label{eq:finite_empirical}
\end{align}
for each arm $\bx\in \bbD$. Another difference is, at each time step $t$, GENTRY chooses with probability $\epsilon_t$ a random arm from $\bbD_k$ rather than $\bbD$.  We define two different settings of $\epsilon_t$ as follows:
\begin{align}
  \epsilon_t^f = \min\left\{1,\frac{\alpha  k}{t}\right\} \text{ and } \epsilon_t^i = \min\left\{1,\frac{\alpha  k}{t^{1/3}}\right\}, 
  \label{eq:epsilont_i}
\end{align}
to handle the finite- and infinite-arm cases, respectively. The hyperparameter $\alpha>0$ is a fixed positive constant. The GENTRY strategy is summarized in \cref{algo:GENTRY}.

\begin{algorithm}[!htb]
\caption{GENTRY}
\begin{algorithmic}[1]
  \label{algo:GENTRY}
  \REQUIRE Set $\epsilon_t$ to $\epsilon_t^f$ or $\epsilon_t^i$ according to whether the problem is finite-arm or infinite-arm. Set hyperparameters $\alpha>0$ and $\lambda>0$. 
\STATE Set $t=1$.
\LOOP
\STATE Set $\epsilon_t$ using \cref{eq:epsilont_i}. 
\STATE With probability $1-\epsilon_t$, choose $\argmax_{\bx\in \bbD} \bar{r}_{t-1}(\bx)$ and with probability $\epsilon_t$ choose a random arm from $\bbD_k$.
\STATE Update $\bar{r}(\bx_t)$ using \cref{eq:finite_empirical}. 
\STATE Set $t = t + 1$. \\
\ENDLOOP
\end{algorithmic} 
\end{algorithm}
In the following theorems, we give a sufficient condition on the hyperparameter $\alpha$ for the strategy to achieve $O(\log n)$ for the finite-arm case. We also show that our infinite-arm strategy achieves $O(n^{2/3}(\log n)^{1/2})$ regret for the infinite-arm case. The empirical impact of $\alpha$ is further studied in \cref{sec:simulation}. Parameter $\lambda$ can be set as a moderate value like $Z^2$.

\subsection{Regret Analysis}\label{ssec:regret}
\subsubsection{Finitely Many Arms}\label{sssec:finite}
\begin{Theorem}\label{theorem:GENTRY_finite}
Suppose \cref{as:as1,as:as2} hold, $\bbD$ contains finitely many arms, 
\begin{align}
\alpha >  \max\braces*{\frac{24dZ^2\vartheta^2}{\min_{\bx\ne X_\bbU^*} \Delta_\bx^2\delta_{\bbD_k}},10}
\end{align}
where  $\Delta_\bx = \ip{P_{\bbU}(X_{\bbU}^*(\btheta))- P_\bbU(\bx)}{P_\bbU(\btheta)}$, and $\delta_{\bbD_k}$ is a constant that depends on $\bbD_k$. Then GENTRY has cumulative projection regret of order $O(|\bbD|\log n)$, where $n$ is the number of time steps and $|\bbD|$ is the number of arms.
\end{Theorem}

\cref{theorem:GENTRY_finite} shows that if we choose $\alpha$ to be sufficiently large, GENTRY achieves order optimal regret. 

\begin{IEEEproof} \label{proof1}
For any arm $\bx\in \bbD$, let the random variable $N_t(\bx)$ denote the total number of time steps within the first $t$ time steps that the arm $\bx$ was chosen, i.e., $N_n(\bx) = \sum_{t=1}^{n} \indicator{X_t = \bx}$. Since 
\begin{align}
  \E[ R_\bbU(n)]& =\E \sum_{t=1}^{n} \bigg(\ip{P_{\bbU}(X_{\bbU}^*(\btheta))}{P_\bbU(\btheta)} - \ip{P_\bbU(X_t)}{P_\bbU(\btheta)}\bigg)\nonumber \\
   &= \sum_{\bx\in \bbD} \E[N_n(\bx)]\Delta_\bx \nonumber \\
  &= \sum_{\bx\in \bbD}\Delta_\bx \sum_{t=1}^{n}\P(X_t = \bx), 
  \label{eq:decompose_R}
\end{align}
where $\Delta_\bx = \ip{P_{\bbU}(X_{\bbU}^*(\btheta))- P_\bbU(\bx)}{P_\bbU(\btheta)}$ is the projection regret of arm $\bx$, it is sufficient to show that the probability $\P(X_t = \bx)= O(t^{-1})$ for all suboptimal $\bx\in \bbD\backslash\{X_\bbU^*\}$. We have
\begin{align}
  \P(X_t = \bx)  \le  \frac{\epsilon_t}{ k} + (1-\epsilon_t)\P(\bar{r}_{t-1}(\bx) \ge \bar{r}_{t-1}(X_{\bbU}^*(\btheta))),
  \label{eq:PX_t}
\end{align}
and 
\begin{align}
  &\P(\bar{r}_{t-1}(\bx) \ge \bar{r}_{t-1}(X_{\bbU}^*(\btheta)))\nn
	&= P(\bar{r}_{t-1}(\bx) \ge \bar{r}_{t-1}(X_{\bbU}^*(\btheta)), \bar{r}_{t-1}(\bx) \ge \ip{P_\bbU(\bx)}{P_\bbU(\btheta)} + \frac{\Delta_\bx}{2}) \nn
	&\hspace{0.3cm}+ \P(\bar{r}_{t-1}(\bx) \ge \bar{r}_{t-1}(X_{\bbU}^*(\btheta)), \bar{r}_{t-1}(\bx) < \ip{P_\bbU(\bx)}{P_\bbU(\btheta)} + \frac{\Delta_\bx}{2}) \nn
  & \le \P(\bar{r}_{t-1}(\bx) \ge \ip{P_\bbU(\bx)}{P_\bbU(\btheta)} + \frac{\Delta_\bx}{2}) \nonumber \\ 
	&\hspace{0.3cm}+ \P(\bar{r}_{t-1}(X_{\bbU}^*(\btheta)) \le \ip{P_\bbU(X_{\bbU}^*(\btheta))}{P_\bbU(\btheta)}-\frac{\Delta_\bx}{2}). 
  \label{eq:Pr_t}
\end{align}
We next bound the first term on the right-hand side of \cref{eq:Pr_t}. The analysis for the second term is similar.

From the proof of Theorem $2$ in \cite{abbasi2011improved}, we have that with probability at least $1-\delta$, for all $t\ge 0$ and $\bx\in \Real^d$, 
\begin{align}
  \left\vert\ip{\bx}{\hat{\btheta}_t - \btheta}\right\vert \le \|\bx\|_{\bV_t^{-1}}\left(\vartheta\sqrt{d\log\left( \frac{1+t}{\delta}\right)}+\lambda^{1/2}S\right).
  \label{eq:concentration}
\end{align}

Let $\beta_t = \vartheta \sqrt{3d \log (1+t)}+\lambda^{1/2}S$. Since $\ip{P_\bbU(\bx)}{P_\bbU(\btheta)} = \ip{P_\bbU(\bx)}{\btheta}$ for the orthogonal projection operator $P_\bbU$, by setting $\delta = \ofrac{t^2}$, we have 

\begin{align*}
\vartheta\sqrt{d\log\left( \frac{1+t}{\delta}\right)}+\lambda^{1/2}S 
&= \vartheta\sqrt{d\log\left( t^2(t+1)\right)}+\lambda^{1/2}S \\
&< \beta_t,
\end{align*}
so that for all $\bx\in \bbD\backslash\{X_\bbU^*\}$, it follows from \cref{eq:concentration} that
\begin{align}
& \P(\bar{r}_t(\bx)-\ip{P_\bbU(\bx)}{P_\bbU(\btheta)} >\beta_t \|P_\bbU(\bx)\|_{\bV_t^{-1}}) \nonumber \\
&\le \mathbb{P} \bigg(\ip{P_\bbU(\bx)} {\hat{\btheta}_t} - \ip{P_\bbU(\bx)}{\btheta}>\nonumber \\
&\hspace{2cm}  \left(\vartheta \sqrt{d\log \left( t^2(t+1)\right)} +\lambda^{1/2}S \right)\|P_\bbU(\bx)\|_{\bV_t^{-1}}\bigg).\nonumber \\
& \le \ofrac{t^2}.\label{eq:concentration2}
\end{align}
We then have
\begin{align}
 &\P(\bar{r}_{t}(\bx) \ge \ip{P_\bbU(\bx)}{P_\bbU(\btheta)}+\frac{\Delta_\bx}{2}) \nonumber \\
 &= \P(\bar{r}_{t}(\bx) \ge \ip{P_\bbU(\bx)}{P_\bbU(\btheta)}+\frac{\Delta_\bx}{2}){\frac{\Delta_\bx}{2}> \beta_t\|P_\bbU(\bx)\|_{\bV_t^{-1}}}\nonumber \\
 & \hspace{1cm} \cdot \P(\frac{\Delta_\bx}{2}> \beta_t\|P_\bbU(\bx)\|_{\bV_t^{-1}})\nonumber \\
 & \hspace{0.0cm} + \P(\bar{r}_{t}(\bx) \ge \ip{P_\bbU(\bx)}{P_\bbU(\btheta)}+\frac{\Delta_\bx}{2}){ \frac{\Delta_\bx}{2}\le \beta_t\|P_\bbU(\bx)\|_{\bV_t^{-1}}} \nonumber \\
 & \hspace{1cm} \cdot \P(\frac{\Delta_\bx}{2}\le \beta_t\|P_\bbU(\bx)\|_{\bV_t^{-1}})\nonumber \\
 & \le \ofrac{t^2} +  \P(\frac{\Delta_\bx}{2}\le \beta_t\|P_\bbU(\bx)\|_{\bV_t^{-1}}).\label{eq:Pr_t_1}
\end{align}

We next show that when the hyperparameter $\alpha$ is sufficiently large, 
\begin{align*}
  \P(\frac{\Delta_\bx}{2}\le \beta_t\|P_\bbU(\bx)\|_{\bV_t^{-1}})= O(t^{-1}).
 \end{align*}


Since
\begin{align*}
  \bV_t &=  \lambda \bI + \sum_{\bx\in \bbD} N_t(\bx)\bx\bx\T  \\
  &=  \lambda P_{\text{span}(\bbD)^\perp}(\bI) +\lambda P_{\text{span}(\bbD)}(\bI)+ \sum_{\bx\in \bbD} N_t(\bx)\bx\bx\T,
\end{align*}
it can be shown by induction that the eigenvectors of $\bV_t$ can be divided into two groups, one in which all the eigenvectors are in $\text{span}(\bbD)$, and another in which all the eigenvectors are in $\text{span}(\bbD)^\perp$. Let $\lambda^\bbD_{t,\text{min}}$ be the smallest eigenvalue of $\bV_t$ in $\text{span}(\bbD)$ and $\lambda^{\bbD_k}_{t,\text{min}}$ be the smallest eigenvalue of $(\lambda \bI+\sum_{\bx\in \bbD_k} N_t(\bx) \bx\bx\T)$ in $\text{span}(\bbD)$. If we define
\begin{align*}
  \by_1=\argmin_{\by\in \text{span}(\bbD), \|\by\|^2_2=1}\ip{\by}{\bV_t\by},
 \end{align*}
 we then have 
\begin{align}
  \lambda^\bbD_{t,\text{min}} &= \by_1\T\bV_t\by_1\nonumber \\
  &\ge \by_1\T(\lambda \bI+ \sum_{\bx\in \bbD_k}N_t(\bx)\bx\bx\T)\by_1  \nonumber \\
  &\ge \min_{\by\in\text{span}(\bbD),\|\by\|^2_2=1}\by\T(\lambda \bI + \sum_{\bx\in \bbD_k} N_t(\bx)\bx\bx\T)\by\nonumber \\
  &= \lambda_{t,\text{min}}^{\bbD_k}\
 \label{eq:eiginequality}
\end{align}
Because  $P_\bbU(\bx) \in \text{span}(\bbD)$, using \cref{eq:eiginequality}, we obtain
\begin{align}
  &\P(\frac{\Delta_\bx}{2} \le \beta_t\|P_\bbU(\bx)\|_{\bV_t^{-1}}) \nonumber \\
  &\le \P(\left(\frac{\Delta_\bx}{2\beta_t}\right)^2 \le \frac{\|P_\bbU(\bx)\|^2_2}{\lambda^\bbD_{t,\text{min}}})\nonumber \\
&=  \P(\lambda_{t,\text{min}}^\bbD \le \frac{4\beta_t^2\|P_\bbU(\bx)\|^2_2}{\Delta_\bx^2})\nonumber \\
&\le \P(\lambda_{t,\text{min}}^{\bbD_k} \le \frac{4\beta_t^2 Z^2}{\min_{\bx\ne X_\bbU^*} \Delta_\bx^2}) 
  \label{eq:mineig}
\end{align}

Consider the function
\begin{align}
  f: \{\by:\by\in\text{span}(\bbD),\|\by\|^2_2=1\}\mapsto \by\T(\sum_{\bx\in \bbD_k} \bx\bx\T)\by.
\end{align}
We have $f(\by)>0$ since $\bbD_k$ contains $k$ linearly independent arms.
We also have $\delta_{\bbD_k} := \inf f(\by) > 0$ since $f(\by)$ is a continuous function on a compact set.
Define the event
\begin{align*}
  \left\{ \delta_{\bbD_k}N_t(\bx) > \frac{4\beta_t^2Z^2}{\min_{\bx\ne X_\bbU^*} \Delta_\bx^2}, \text{ for all } \bx\in \bbD_k\right\}.
\end{align*}
Under this event, using the last equality in \cref{eq:eiginequality}, we obtain $\lambda_{t,\text{min}}^{\bbD_k} >\frac{4\beta_t^2Z^2}{\min_{\bx\ne X_\bbU^*}\Delta_\bx^2}$. Therefore, we have
\begin{align}
  & \P(\lambda_{t,\text{min}}^{\bbD_k}\le \frac{4\beta_t^2Z^2}{\min_{\bx\ne X_\bbU^*} \Delta_\bx^2})\nonumber \\
 & \le \sum_{\bx \in \bbD_k}\P(N_t(\bx)\le  \frac{4\beta_t^2Z^2}{\min_{\bx\ne X_\bbU^*} \Delta_\bx^2 \delta_{\bbD_k}})\nonumber \\
 & \le \sum_{\bx \in \bbD_k}\P(\tilde{N}_t(\bx)\le  \frac{4\beta_t^2Z^2}{\min_{\bx\ne X_\bbU^*} \Delta_\bx^2 \delta_{\bbD_k}}),\label{eq:eigbound}
\end{align}
where $\tilde{N}_t(\bx) \leq N_t(\bx)$ is the number of times that arm $\bx\in \bbD_k$ was chosen \emph{randomly} during the first $t$ time steps. From the proof of Theorem 3 in \cite{auer2002finite} (where Bernstein's inequality \cite{Pollard} was used), we have
\begin{align}
  \P(\tilde{N}_t(\bx) \le \ofrac{2 k}\sum_{i = 1}^{t}\epsilon_i)\le \exp{ \left(-\frac{1}{10 k}\sum_{i = 1}^{t}\epsilon_i\right)} \label{eq:bernstein}
\end{align}
For $t\ge t' = \ceil{\alpha k}$, $\epsilon_t=\epsilon_t^f  = \frac{\alpha  k}{t}$, we then obtain 
\begin{align}
\ofrac{2 k}\sum_{i = 1}^{t}\epsilon_i 
& = \ofrac{2 k} \sum_{i=1}^{t'} \epsilon_i+\ofrac{2 k} \sum_{i= t'+1}^{t}\epsilon_i \nn
& \ge \frac{\alpha}{2} + \ofrac{2 k}\int_{t'+1}^{t+1} \frac{\alpha  k}{t} \ud x \nn
&\ge \frac{\alpha}{2} + \frac{\alpha}{2}\log \frac{t+1}{\alpha k+2} \nn
&=  \frac{\alpha}{2}\log \frac{(t+1)e}{\alpha k+2}, \label{eq:sumepsilon}
\end{align}
where $e$ is Euler's number. If $\frac{\alpha}{2}\log \frac{(t+1)e}{\alpha k+2}\ge  \frac{4\beta_t^2Z^2}{\min_{\bx\ne X_\bbU^*} \Delta_\bx^2 \delta_{\bbD_k}}$, from \cref{eq:eigbound}, we have 
\begin{align}
&  \P(\lambda_{t,\text{min}}^{\bbD_k}\le \frac{4\beta_t^2Z^2}{\min_{\bx\ne X_\bbU^*} \Delta_\bx^2})\nn
&\le \sum_{\bx\in \bbD_k} \P(\tilde{N}_t(\bx)\le  \frac{4\beta_t^2Z^2}{\min_{\bx\ne X_\bbU^*} \Delta_\bx^2 \delta_{\bbD_k}}) \nonumber \\
&\le \sum_{\bx\in \bbD_k} \P(\tilde{N}_t(\bx)\le   \frac{\alpha}{2}\log \frac{(t+1)e}{\alpha k+2}) \nn
& \le \sum_{\bx\in \bbD_k}  \P(\tilde{N}_t(\bx)\le  \ofrac{2 k}\sum_{i = 1}^{t}\epsilon_i)\nn
& \le \sum_{\bx\in \bbD_k}  \exp\left(-\ofrac{10 k}\sum_{i = 1}^{t}\epsilon_i\right) \nn
& \le k \left(\frac{\alpha k+2}{(t+1)e}\right)^{\frac{\alpha}{10}},\label{eq:Ot}
\end{align}
where the third inequality follows from \cref{eq:sumepsilon}, the penultimate inequality follows from \cref{eq:bernstein}, and the last inequality from \cref{eq:sumepsilon}.

Recall that $\beta_t = \vartheta \sqrt{3d \log (t+1)}+\lambda^{1/2}S$. Therefore, if $\alpha > \max\{\frac{24dZ^2\vartheta^2}{\min_{\bx\ne X_\bbU^*} \Delta_\bx^2\delta_{\bbD_k}},10\}$, when $t$ is sufficiently large, the condition $\frac{\alpha}{2}\log \frac{(t+1)e}{\alpha k+2}\ge  \frac{4\beta_t^2Z^2}{\min_{\bx\ne X_\bbU^*} \Delta_\bx^2 \delta_{\bbD_k}}$ is satisfied. We have
\begin{align}
 &\P(\bar{r}_{t}(\bx) \ge \ip{ P_\bbU(\bx)}{P_\bbU(\btheta)}+\frac{\Delta_\bx}{2}) \nn
 & \le \ofrac{t^2} +  \P(\frac{\Delta_\bx}{2}\le \beta_t\|P_\bbU(\bx)\|_{\bV_t^{-1}})\nn
 &\le \ofrac{t^2}+ \P(\lambda_{t,\text{min}}^{\bbD_k} \le \frac{4\beta_t^2 Z^2}{\min_{\bx\ne X_\bbU^*} \Delta_\bx^2})\nn
 & \le \ofrac{t^2} + k \left(\frac{\alpha k+2}{(t+1)e}\right)^{\frac{\alpha}{10}}\nn
 & = o(t^{-1}),
 \label{eq:Pr_t_2}
\end{align}
where the first inequality follows from \cref{eq:Pr_t_1}, the second inequality from \cref{eq:mineig} and the last inequality from \cref{eq:Ot}.

A similar argument shows that $\P(\bar{r}_{t}(\bx) \le \ip{P_\bbU(X_\bbU^*)}{P_\bbU(\btheta)}-\frac{\Delta_\bx}{2})= o(t^{-1})$.
Then, using \cref{eq:PX_t,eq:Pr_t}, we have
\begin{align}
  \P(X_t = \bx) & \le \frac{\alpha}{t}+o(t^{-1}) = O(t^{-1}),
  \label{eq:PX_t_2}
\end{align}
From \cref{eq:decompose_R}, we conclude that 
\begin{align*}
\E[R_\bbU(n)] = O(|\bbD|\log n).
\end{align*}
The proof of \cref{theorem:GENTRY_finite} is now complete. 
\end{IEEEproof}

\subsubsection{Infinitely Many Arms}\label{sssec:infinite}

\begin{Theorem}\label{theorem:GENTRY_infinite}
Suppose \cref{as:as1,as:as2} hold, and $\bbD$ is a compact set containing infinitely many arms. Then, GENTRY has cumulative projection regret of order $O\left(\max\left\{k,\sqrt{d}\right\}n^{2/3}(\log n)^{1/2}\right)$, where $n$ is the number of time steps.
\end{Theorem}
\begin{IEEEproof}
The contribution to the projection regret at time step $t$ comes from two cases: 1) when $X_t$ is chosen randomly, or 2) when $X_t$ is chosen as $\argmax_{\bx\in \bbD}\bar{r}_{t-1}(\bx)$. We denote these two event as $E_t$ and $E_t^{c}$ respectively.
\begin{align}
  & \E[R_\bbU(n)] \nonumber \\
 & = \E[\sum_{t=1}^{n} \ip{P_{\bbU}(X_\bbU^*)-P_\bbU(X_t)}{P_\bbU(\btheta)}] \nonumber \\
 & = \sum_{t=1}^{n}\big( \epsilon_t \E[\ip{P_{\bbU}(X_\bbU^*) - P_\bbU(X_t)}{P_\bbU(\btheta)}]{ E_t}\nonumber \\
 & \hspace{0.3cm}+(1-\epsilon_t) \E[\ip{P_{\bbU}(X_\bbU^*) - P_\bbU(X_t)}{P_\bbU(\btheta)}]{E_t^c} \big)\nonumber \\
 & \le\sum_{t=1}^{n}\left(2ZS \epsilon_t +\E[\ip{P_{\bbU}(X_\bbU^*) - P_\bbU(X_t)}{P_\bbU(\btheta)}]{ E_t^c}\right),
  \label{eq:R_inf}
\end{align}
where the inequality follows from \cref{as:as1}, $\|P_\bbU(\bx)\|_2\le \|\bx\|_2$ for all $\bx$, and the Cauchy-Schwarz inequality.

We next derive an upper bound for the second term $\E[\ip{P_{\bbU}(X_\bbU^*) - P_\bbU(X_t)}{P_\bbU(\btheta)}]{ E_t^c}$ in the sum on the right-hand side of \cref{eq:R_inf}.

Under the event $E_t^c$, we have
\begin{align}
  & \ip{P_{\bbU}(X_\bbU^*) - P_\bbU(X_t)}{P_\bbU(\btheta)} \nonumber \\
  & = \ip{P_{\bbU}(X_\bbU^*) - P_\bbU(X_t)}{P_\bbU(\hat{\btheta}_{t-1})}\nonumber \\
  & \hspace{0.3cm} + \ip{P_{\bbU}(X_\bbU^*) - P_\bbU(X_t)}{P_\bbU(\btheta)- P_\bbU(\hat{\btheta}_{t-1})} \nonumber \\
  &\le \ip{P_{\bbU}(X_\bbU^*) - P_\bbU(X_t)}{P_\bbU(\btheta)- P_\bbU(\hat{\btheta}_{t-1})}, 
  \label{eq:regret2pro}
\end{align}
where the inequality follows from $X_t = \argmax_{\bx\in \bbD} \bar{r}_{t-1}(\bx)$.

We use the same $\beta_t$ defined in the proof of \cref{theorem:GENTRY_finite}. For any $t\ge 2$, let the event  
\begin{align}
  F_t = \left\{ \ip{P_{\bbU}(X_\bbU^*-X_t)}{P_\bbU(\btheta)}\nonumber \ge \|P_{\bbU}(X_\bbU^*-X_t)\|_{\bV_{t-1}^{-1}}\beta_{t-1} \right\}.
\end{align}
From \cref{eq:regret2pro,eq:concentration}, since $\ip{P_\bbU(\bx)}{P_\bbU(\btheta)} = \ip{P_\bbU(\bx) }{\btheta}$ for the orthogonal projection operator $P_\bbU$,  we have 
\begin{align}
  & \P(F_{t+1}) \nonumber\\
  &  \le \P\bigg( \ip{P_\bbU\left(X_\bbU^*-X_{t+1}\right)}{\btheta-\hat{\btheta}_{t}}\ge  \|P_{\bbU}(X_\bbU^*- X_{t+1})\|_{\bV_{t}^{-1}}\beta_{t}\bigg) \nonumber \\
& \le \P\bigg(  \ip{P_\bbU\left(X_\bbU^*-X_{t+1}\right)}{\btheta-\hat{\btheta}_{t}} \nonumber \\ &\hspace{0.5cm}\ge  \| P_\bbU\left(X_\bbU^*-X_{t+1}\right)\|_{\bV_{t}^{-1}}\bigg(\vartheta\sqrt{d\log\left( \frac{1+t}{t^2}\right)}+\lambda^{1/2}S\bigg)\bigg)\nonumber \\
  &\le \ofrac{t^2},\label{eq:useconcentration}
\end{align}
where the final inequality follows from the concentration inequality \cref{eq:concentration} with $\delta = \ofrac{t^2}$. We then have 
\begin{align}
  & \E[\ip{P_{\bbU}(X_\bbU^*) - P_\bbU(X_{t+1})}{P_\bbU(\btheta)}]{E_{t+1}^c}\nonumber \\
  & = \E[\ip{P_{\bbU}(X_\bbU^*-X_{t+1})}{P_\bbU(\btheta)}]{E_{t+1}^c, F_{t+1}}\P(F_{t+1})\nonumber \\
  &\hspace{0.2cm} +\E[ \ip{P_{\bbU}(X_\bbU^*-X_{t+1})}{P_\bbU(\btheta)}\indicatore{F_{t+1}^c}]{E_{t+1}^c}\nonumber \\
  & \le \frac{2ZS}{t^2}+ \E[ \|P_{\bbU}(X_\bbU^*-X_{t+1})\|_{\bV_t^{-1}}\beta_t]{E_{t+1}^c}\label{eq:bounddifference_0} \\
  & \le \frac{2ZS}{t^2}+\E [2Z\beta_t\sqrt{\frac{1}{\lambda_{t,\text{min}}^\bbD}}]{E_{t+1}^c}\label{eq:bounddifference_1}\\
& \le \frac{2ZS}{t^2} + \E[2Z\beta_t\sqrt{\frac{1}{\lambda_{t,\text{min}}^\bbD}}] \nn 
& \le \frac{2ZS}{t^2} + 2Z\beta_t\sqrt{\E[\frac{1}{\lambda_{t,\text{min}}^\bbD}]}\label{eq:bounddifference_3}\\\
& \le \frac{2ZS}{t^2} + 2Z\beta_t\sqrt{\E[\frac{1}{\lambda_{t,\text{min}}^{\bbD_k}}]},
\label{eq:bounddifference}
\end{align}
where 
\begin{itemize}
	\item \cref{eq:bounddifference_0} follows from \cref{eq:useconcentration};
	\item \cref{eq:bounddifference_1} follows because $\|P_{\bbU}(X_\bbU^*-X_t)\|^2_{\bV_t^{-1}} \le \|P_\bbU(X_\bbU^*-X_t)\|^2_2/\lambda^\bbD_{t,\text{min}}$;
	\item \cref{eq:bounddifference_3} follows from Jensen's inequality; and
	\item \cref{eq:bounddifference} follows from \cref{eq:eiginequality}.
\end{itemize}

We next show that $ \E[\frac{1}{\lambda_{t,\text{min}}^{\bbD_k}}] =O(t^{-2/3})$. We first show $\P(\lambda_{t,\text{min}}^{\bbD_k} < t^{2/3}) \le O(e^{-t^{2/3}})$. 
  Similar to what we have done in \cref{eq:eigbound}, define the following event:
\begin{align*}
\left\{ N_t(\bx) > \frac{\alpha}{2}t^{2/3}, \text{ for all } \bx\in \bbD_k\right\}.
\end{align*}
Under this event, from the last equality in \cref{eq:eiginequality}, we get $\lambda_{t,\text{min}}^{\bbD_k} >\frac{\alpha}{2}t^{2/3}\delta_{\bbD_k}$. Therefore, we have
\begin{align}
  & \P(\lambda_{t,\text{min}}^{\bbD_k}\le \frac{\alpha}{2}t^{2/3}\delta_{\bbD_k})\nonumber \\
 & \le \sum_{\bx \in \bbD_k}\P(N_t(\bx)\le   \frac{\alpha}{2}t^{2/3})\nonumber \\
 & \le \sum_{\bx \in \bbD_k}\P(\tilde{N}_t(\bx)\le  \frac{\alpha}{2}t^{2/3}),\label{eq:infeigbound}
\end{align}
where $\tilde{N}_t(\bx)$ is the number of times that arm $\bx\in \bbD_k$ was chosen \emph{randomly} during the first $t$ time steps. Recall from \cref{eq:bernstein} that 
\begin{align}
  \P(\tilde{N}_t(\bx) \le \ofrac{2 k}\sum_{i = 1}^{t}\epsilon_i)\le \exp{ \left(-\frac{1}{10 k}\sum_{i = 1}^{t}\epsilon_i\right)}
  \label{eq:temp}
\end{align}
For $t\ge t' =\ceil{( \alpha  k)^3}$, $\epsilon_t =\epsilon_t^i = \frac{\alpha  k}{t^{1/3}}$ and we obtain
\begin{align}
\ofrac{2 k}\sum_{i = 1}^{t}\epsilon_i & = \ofrac{2 k} \sum_{i=1}^{t'} \epsilon_i+\ofrac{2 k} \sum_{i= t'+1}^{t}\epsilon_i \nn
& \ge \ofrac{2 k}t' + \ofrac{2 k}\int_{t'+1}^{t+1} \frac{\alpha  k}{t^{1/3}} \ud x \nn
& \ge  \frac{\alpha^3 k^2}{2} + \frac{3\alpha}{4}t^{2/3}-\frac{3\alpha}{4}(\alpha^3 k^3+2)^{2/3}. \nn
& \ge \frac{\alpha}{2}t^{2/3},
\label{eq:infsumepsilon}
\end{align}
when $t$ is sufficiently large.
It follows from \cref{eq:infeigbound,eq:temp,eq:infsumepsilon} that
\begin{align}
  & \P(\lambda_{t,\text{min}}^{\bbD_k}\le \frac{\alpha}{2}t^{2/3}\delta_{\bbD_k})\nonumber \\
  & \le \sum_{\bx \in \bbD_k}\P(\tilde{N}_t(\bx)\le  \frac{\alpha}{2}t^{2/3}),\nonumber \\
  & \le \sum_{\bx \in \bbD_k}\P(\tilde{N}_t(\bx) \le \ofrac{2 k}\sum_{i = 1}^{t}\epsilon_i), \nonumber \\
  & \le k \exp\left(-\frac{\alpha}{10}t^{2/3}\right)
  \label{eq:infeig}
\end{align}
Now, we can conclude that when $t$ is sufficiently large,
\begin{align}
  &\E[\ofrac{\lambda_{t,\min}^{\bbD_k}}]\nonumber\\
  &=\E[\ofrac{\lambda_{t,\min}^{\bbD_k}}]{\lambda_{t,\text{min}}^{\bbD_k}> \frac{\alpha}{2}t^{2/3}\delta_{\bbD_k}}\P(\lambda_{t,\text{min}}^{\bbD_k}> \frac{\alpha}{2}t^{2/3}\delta_{\bbD_k})\nonumber \\
  & \hspace{0.3cm} +\E[\ofrac{\lambda_{t,\min}^{\bbD_k}}]{\lambda_{t,\text{min}}^{\bbD_k}\le \frac{\alpha}{2}t^{2/3}\delta_{\bbD_k}}\P(\lambda_{t,\text{min}}^{\bbD_k}\le \frac{\alpha}{2}t^{\frac{2}{3}}\delta_{\bbD_k})\nonumber \\
  &\le \frac{2t^{-2/3}}{\alpha\delta_{\bbD_k}}\P(\lambda_{t,\text{min}}^{\bbD_k}> \frac{\alpha}{2}t^{2/3}\delta_{\bbD_k})+ \frac{k}{\lambda}\exp\left(-\frac{\alpha}{10}t^{2/3}\right)\nn
  &\le  \frac{3t^{-2/3}}{\alpha\delta_{\bbD_k}}, 
  \label{eq:almostdone}
\end{align}
where the penultimate inequality follows from \cref{eq:eiginequality}.

From \cref{eq:bounddifference,eq:almostdone}, when $t$ is sufficiently large, we have
\begin{align}
  & \E[\ip{P_{\bbU}(X_\bbU^*) - P_\bbU(X_{t+1})}{P_\bbU(\btheta)}]{E_{t+1}^c}\nonumber \\
& \le \frac{2ZS}{t^2} +2Z\beta_t\sqrt{\E[\frac{1}{\lambda_{t,\text{min}}^{\bbD_k}}]}\nonumber \\ 
&\le \frac{2ZS}{t^2}+2Z\beta_t\sqrt{\frac{3}{\alpha\delta_{\bbD_k}}}t^{-\ofrac{3}}
  \label{eq:onemorestep}
\end{align}
Finally, from \cref{eq:R_inf}, we have
\begin{align}
  &\E[R_\bbU(n)]\nonumber \\
  & \le\sum_{t=1}^{n} \left(2ZS \epsilon_t +\E[\ip{P_{\bbU}(X_\bbU^*) - P_\bbU(X_t)}{P_\bbU(\btheta)}]{ E_t^c} \right)\nonumber \\
  &\le \sum_{t=1}^{n}  \left(\frac{2ZS\alpha k}{t^{1/3}}+ 2Z\beta_t\sqrt{\frac{3}{\alpha\delta_{\bbD_k}}}t^{-\ofrac{3}}\right) + C_1\nonumber \\
  &\le \int_{1}^{n} \left(\frac{C_2 k}{x^{1/3}}+ 6Z\vartheta\sqrt{\frac{d\log x}{\alpha\delta_{\bbD_k}}}x^{-\ofrac{3}}\right) \ud x+C_3\nonumber \\
& = O\left(\max\left\{k,\sqrt{d}\right\}n^{2/3}(\log n)^{1/2}\right),
  \label{eq:final}
\end{align}
where $C_1$, $C_2$ and $C_3$ are constants. The proof of \cref{theorem:GENTRY_infinite} is now complete.  
\end{IEEEproof}

\subsection{Additional Smoothness Assumption}\label{sec:smoothness}

In this subsection, we show that if an additional smoothness assumption similar to that in \cite{rusmevichientong2010linearly} is made, we can achieve better cumulative projection regret order than $O(n^{2/3}(\log n)^{1/2})$ for the infinite-arm case. 

\begin{Assumption}\label{as:as3}
There exists $J \in \Real^{+}$ such that for all $\btheta_1$, $\btheta_2\in \Real^d$, 
\begin{align}\label{eq:sbar}
&\left\|P_\bbU\left(X_\bbU^*(\btheta_1) -X_\bbU^*(\btheta_2)\right)\right\|_2 \le J\left\|\frac{P_\bbU(\btheta_1)}{\|P_\bbU(\btheta_1)\|} - \frac{P_\bbU(\btheta_1)}{\|P_\bbU(\btheta_1)\|} \right\|_2.
\end{align}
\end{Assumption}
This is to say the projection of the decision set $\bbD$ and all $\btheta \in \Real^d$ onto the subspace $\bbU$ satisfies the $\text{SBAR}(J)$ condition \cite{rusmevichientong2010linearly,polovinkin1996strongly}. For example, if the decision set $\bbD$ is a ball or an elliposid, it satisfies the condition \cref{eq:sbar}. In the proof of \cref{theorem:GENTRY_infinite}, we bound \cref{eq:regret2pro} with $O\left(\beta_t(\E[1/\lambda_{t,\text{min}}^{\bbD_k}])^{1/2}]\right) =O(t^{-1/3}\log t)$. 
With the additional \cref{as:as3}, we can improve this bound, and further get a tighter upper bound for the cumulative projection regret in the following result.

\begin{Theorem}\label{theorem:GENTRYsmooth}
Suppose \cref{as:as1,as:as2,as:as3} hold and $\bbD$ is a compact set with infinitely many arms. Then, GENTRY with $\epsilon_t = \epsilon_t^s := \min\left\{1,\frac{\alpha  k}{\sqrt{t}}\right\}$ has cumulative projection regret of order $O(d\sqrt{n})$, where $n$ is the number of time steps.
\end{Theorem} 
\begin{proof}
The inequalities \cref{eq:R_inf,eq:regret2pro} from the proof of \cref{theorem:GENTRY_infinite} still hold. We have 
  \begin{align}
    & \E[ \ip{P_{\bbU}(X_\bbU^*) - P_{\bbU}(X_t)}{P_{\bbU}(\btheta)}\mid E_t^c] \nonumber \\ 
    & \le\E[\ip{P_{\bbU}(X_\bbU^*) - P_\bbU(X_t)}{P_{\bbU}(\btheta)- P_{\bbU}(\hat{\btheta}_{t-1})}]{E_t^c}\nonumber \\ 
    & \le\E[\left\Vert P_{\bbU}(X_\bbU^*) - P_{\bbU}(X_t) \right\Vert_2 \left\Vert P_{\bbU}(\btheta)- P_{\bbU}(\hat{\btheta}_{t-1}) \right\Vert_2]{E_t^c} \nonumber \\ 
    & \le \E[J \left\|\frac{P_{\bbU}(\btheta)}{\|P_{\bbU}(\btheta)\|} - \frac{P_{\bbU}(\hat{\btheta}_{t-1})}{\|P_{\bbU}(\hat{\btheta}_{t-1})\|}\right\|_2\left\Vert P_{\bbU}(\btheta)- P_{\bbU}(\hat{\btheta}_{t-1}) \right\Vert_2]{E_t^c} \nonumber \\  
    &\le  \E[\frac{2J\left\Vert P_{\bbU}(\btheta)- P_{\bbU}(\hat{\btheta}_{t-1}) \right\Vert_2^2}{\|P_{\bbU}(\btheta)\|}]{E_t^c} \nonumber \\ 
    &=  \E[\frac{2J\left\Vert P_{\bbU}(\btheta-\hat{\btheta}_{t-1}) \right\Vert_2^2}{\|P_{\bbU}(\btheta)\|}]
\end{align}
where the first inequality follows from \cref{eq:regret2pro}, the second inequality follows from Cauchy-Schwarz inequality, the third inequality follows from \cref{as:as3}, the penultimate inequality follows from Lemma 3.5 in \cite{rusmevichientong2010linearly}, and  the last equality follows from independence.

We next bound  $\E[\left\Vert P_{\bbU}(\btheta-\hat{\btheta}_{t}) \right\Vert_2^2]$. From \cref{eq:estimate},  we have,
\begin{align}
    \hat{\btheta}_t & = \parens*{\lambda \bI+\sum_{i = 1}^{t}X_iX_i\T}^{-1}\sum_{i = 1}^{t} r_iX_i \nonumber \\
    & =  \parens*{\lambda \bI+\sum_{i = 1}^{t}X_iX_i\T}^{-1}\sum_{i = 1}^{t} ( \ip{X_i}{\btheta} + \eta_i)X_i,\nonumber
\end{align}
which yields
\begin{align}
  &\E[\left\Vert P_{\bbU}(\btheta-\hat{\btheta}_{t}) \right\Vert_2^2]\nonumber \\
  &\le \E[\left\Vert P_{\bbU}(\lambda \bV^{-1}_t\btheta) \right\Vert_2^2 + \left\Vert  \sum_{i = 1}^{t}P_{\bbU}(\bV^{-1}_t X_i) \eta_i \right\Vert_2^2]  \nonumber \\
  &\le \lambda^2S^2\E[\ofrac{(\lambda_{t,\text{min}}^\bbD)^2}] + \zeta \E[\sum_{i = 1}^{t}P_{\bbU}(\bV^{-1}_t X_i)\T P_{\bbU}(\bV^{-1}_t X_i)] \nonumber \\
  &\le \lambda^2S^2\E[\ofrac{(\lambda_{t,\text{min}}^\bbD)^2}] + \zeta \E[\sum_{i = 1}^{t}\tr\left(P_{\bbU}(\bV^{-1}_t)\bV^{-1}_t(X_iX_i\T)\right)] \nonumber \\
  &\le \lambda^2S^2\E[\ofrac{(\lambda_{t,\text{min}}^\bbD)^2}] + \zeta \E[\tr\left(P_{\bbU}(\bV^{-1}_t)\bV^{-1}_t (\sum_{i = 1}^{t}X_iX_i\T\right)] \nonumber \\
    &\le \lambda^2S^2\E[\ofrac{(\lambda_{t,\text{min}}^\bbD)^2}] + \zeta \E[d\lambda_{\text{max}}(P_{\bbU}(\bV^{-1}_t))] \nonumber \\
   &\le \lambda^2S^2\E[\ofrac{(\lambda_{t,\text{min}}^\bbD)^2}] + d\zeta \E[\frac{1}{\lambda_{t,\text{min}}^\bbD}] \nonumber \\
   &\le \lambda^2S^2\E[\ofrac{(\lambda_{t,\text{min}}^{\bbD_k})^2}] + d\zeta \E[\ofrac{\lambda_{t,\text{min}}^{\bbD_k}}],\label{eq:expect}
\end{align}
where the second inequality has used the condition that $\eta_t$ are \gls{iid} and the standard result that $\E[\eta_t^2] \le \zeta$ where $\zeta$ is a constant depending on $\vartheta$ since $\eta_t$ has the zero-mean sub-Gaussian distribution with parameter $\vartheta$. 

Using a similar argument as that in the proof of \cref{theorem:GENTRY_infinite}, we can show
\begin{align}
  \E[\frac{1}{\lambda_{t,\text{min}}^{\bbD_k}}]= O(t^{-1/2}), \text{   and  }
  \E[\ofrac{(\lambda_{t,\text{min}}^{\bbD_k})^2}]= O(t^{-1}).
  \label{eq:order_as3}
\end{align}
The proof steps are skipped here for brevity. Finally, using \cref{eq:R_inf,eq:R_inf,eq:expect}, we have
\begin{align}
  &\E[R_\bbU(n)]\nonumber \\
  & \le\sum_{t=1}^{n}2ZS \epsilon_t^s +\E[\ip{P_{\bbU}(X_\bbU^*) - P_\bbU(X_t)}{P_\bbU(\btheta)}]{E_t^c}\nonumber \\
  &\le \sum_{t=1}^{n} dO(t^{-1/2})\nonumber \\
  & = O(d\sqrt{n}).
  \label{eq:final_s}
\end{align}
The proof of \cref{theorem:GENTRYsmooth} is now complete.  
\end{proof}

\section{Simulations and Performance Evaluation}\label{sec:simulation}
In this section, we verify the efficiency of our strategy by performing simulations on synthetic data and a wine quality dataset\cite{cortez2009modeling}. We compare against the following strategies:
\begin{enumerate}[1)]
\item the Uncertainty Ellipsoid (UE) policy \cite{rusmevichientong2010linearly}, 
  which chooses \[\argmax_{\bx\in \bbD} \bar{r}_{t-1}(\bx) +\alpha\sqrt{\log t\min\{d\log t,|\bbD|\}} \|\bx\|_{\bV_{t-1}^{-1}}\] at each time step $t$;  
\item the CorrUpted GReedy StratEgy (CURSE) by replacing $\argmax_{\bx\in \bbD} \bar{r}_{t-1}(\bx)$ with $\argmax_{\bx\in \bbD} \ip{\bx}{\hat{\btheta}_{t-1}}$ in GENTRY, i.e., the corruption in the observed return is not accounted for; and
\item the caRelEss GReEdy sTrategy (REGRET) where the $X_i$ in GENTRY are all replaced by $P_\bbU(X_i)$. To avoid any doubts, the full REGRET is shown in \cref{algo:REGRET}.
\end{enumerate}
 
\begin{algorithm}[!htp]
\caption{REGRET}
  \label{algo:REGRET}
\begin{algorithmic}[1]
  \REQUIRE Set $\epsilon_t$ to $\epsilon_t^f$ or $\epsilon_t^i$ according to the number of arms; set hyper parameter $\alpha>0$ and $\lambda>0$.
\LOOP
\STATE Update $\epsilon_t$. 
\STATE With probability $1-\epsilon_t$, choose 
\begin{align*}
\argmax_{\bx\in \bbD} \ip{P_\bbU(\bx)}{P_\bbU(\tilde{\btheta}_{t-1})},
\end{align*}
where
\begin{align*}
\tilde{\btheta}_t = \left(\sum_{i=1}^{t}P_\bbU(X_i)P_\bbU(X_i)\T+\lambda \bI\right)^{-1}\sum_{i=1}^{t}r_iP_\bbU(X_i),
\end{align*}
and with probability $\epsilon_t$ choose a random arm from $\bbD_k$.
\STATE  Update $\ip{P_\bbU(\bx)}{P_\bbU(\tilde{\btheta}_{t-1})}$. Set $t = t + 1$.\\
\ENDLOOP
\end{algorithmic} 
\end{algorithm}
In each simulation, we perform $2000$ trials, each with $10^4$ time steps. To compare the projection regret performance of different strategies, we compute the average empirical cumulative projection regret over all the trials. For convenience, this average is referred to as the \emph{average cumulative projection regret}.

\subsection{Experiments on Synthetic Data}\label{ssec:synthetic}
In this section, we compare the performance of different strategies using synthetic data.
We use three settings in the simulations:
\begin{enumerate}[(a)]
  \item \label{it:settinga}
    For the finite-arm case, in each trial, we generate the decision set $\bbD\subseteq \Real^d$ containing $K$ arms, in which each arm is associated with a $d$-dimension feature vector. Each dimension of the feature vector is drawn \gls{iid} from the uniform distribution $\calU(-1,1)$. Each dimension of the ground-truth parameter $\btheta$ is also drawn from $\calU(-1,1)$. We next generate the orthogonal projection matrix $P_\bbU = \bA(\bA\T\bA)^{-1}\bA\T$, where $\bA$ is a $d\times u$ ($u<d$) matrix whose elements are generated randomly from $\calU(-1,1)$. The noise $\eta_t$ at each time step is drawn \gls{iid} from the Gaussian distribution $\calN(0,\vartheta^2)$ with variance $\vartheta^2$. The decision set $\bbD$, parameter $\btheta$ and projection matrix $P_\bbU$ are fixed in each trial.
  \item \label{it:settingb}
      We use the same setting as in setting \ref{it:settinga} except that the projection matrix $P_\bbU$ in this setting is a diagonal matrix whose $(i,i)$ entry is $1$ for $i=1,\cdots,u$ and $0$ otherwise. This means that the last $d-u$ dimensions of $\btheta$ are the protected features.
  \item \label{it:settingc}
    For the infinite-arm case, in each trial, the decision set $\bbD$ is limited to a convex set for ease of computation. Specifically, we use the convex set $\bbD$:
    \begin{align}\label{eq:convexD}
      \bbD = \left\{\bx: \sum_{i=1}^{d} \bx(i)\log \bx(i) \le 5 \text{ and } \bx(i) \ge 0\ \forall\ i  \right\},
\end{align}
where $\bx(i)$ is the $i$-th entry of $\bx$, $0\log 0 := 0$, and $\btheta$, $P_\bbU$ and $\eta_t$ are generated the same way as in setting \ref{it:settinga}.

  \end{enumerate}
%

In the following, GENTRY, CURSE, and REGRET use the setting $\epsilon = \epsilon_t^f$ described in \cref{sec:strategy} when in setting \ref{it:settinga} and \ref{it:settingb} for the finite-arm case. Accordingly, $\epsilon = \epsilon_t^i$ is used when in setting \ref{it:settingc} for the infinite-arm case. In all simulations, the parameter $\lambda$ is set to $1$.

\subsubsection{Varying \texorpdfstring{$\alpha$}{alpha}}\label{sssec:vary_alpha}

\begin{figure}[!htb]
    \centering
    \begin{subfigure}[b]{0.9\columnwidth}
        \centering
	\includegraphics[width=\linewidth]{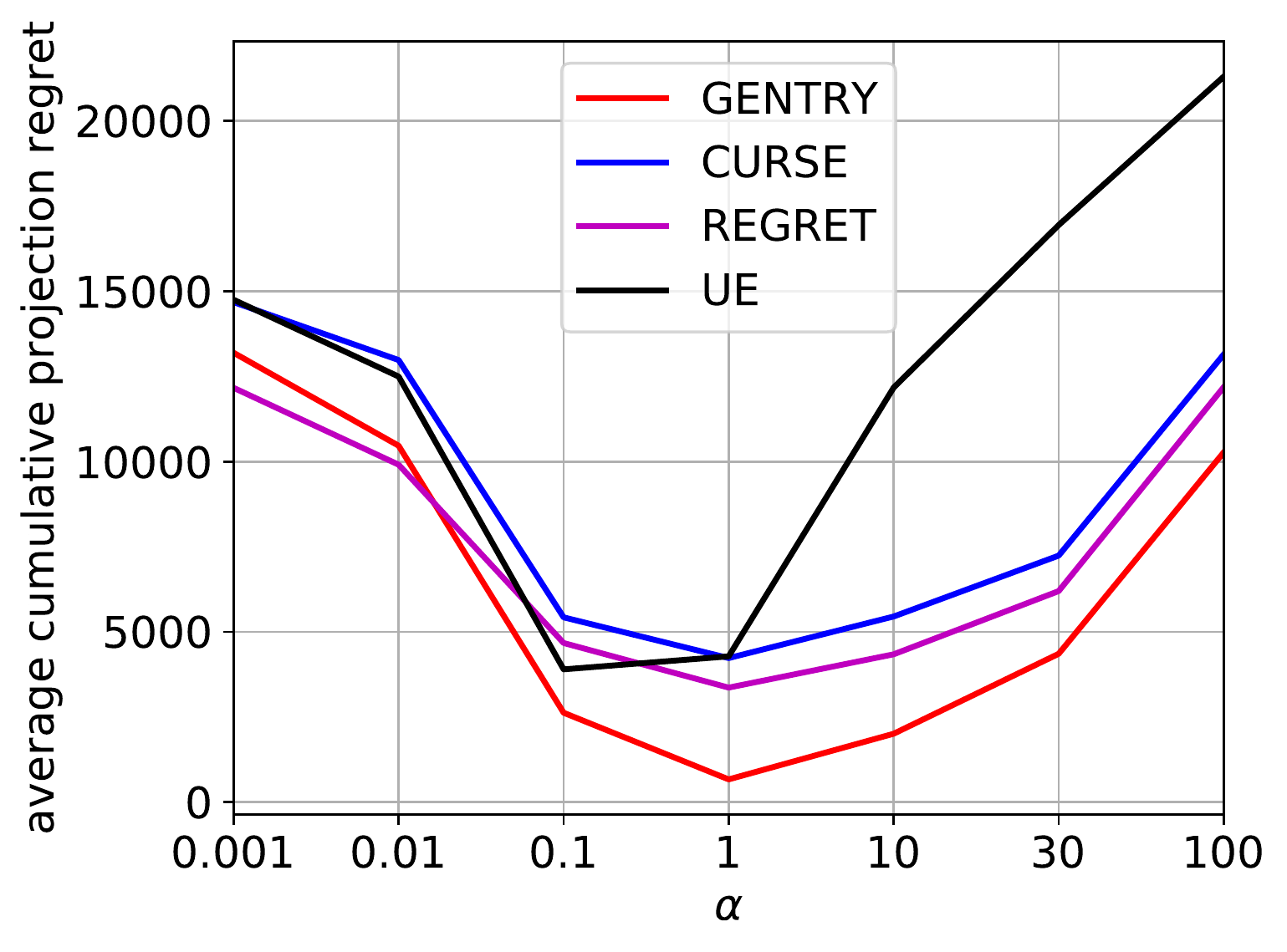}
	\caption{Using setting \ref{it:settinga} with $d=10$, $K=45$, and $u = 5$}
        \label{fig:parameter}
    \end{subfigure}\\
    \begin{subfigure}[b]{0.9\columnwidth}
        \centering
	\includegraphics[width=\linewidth]{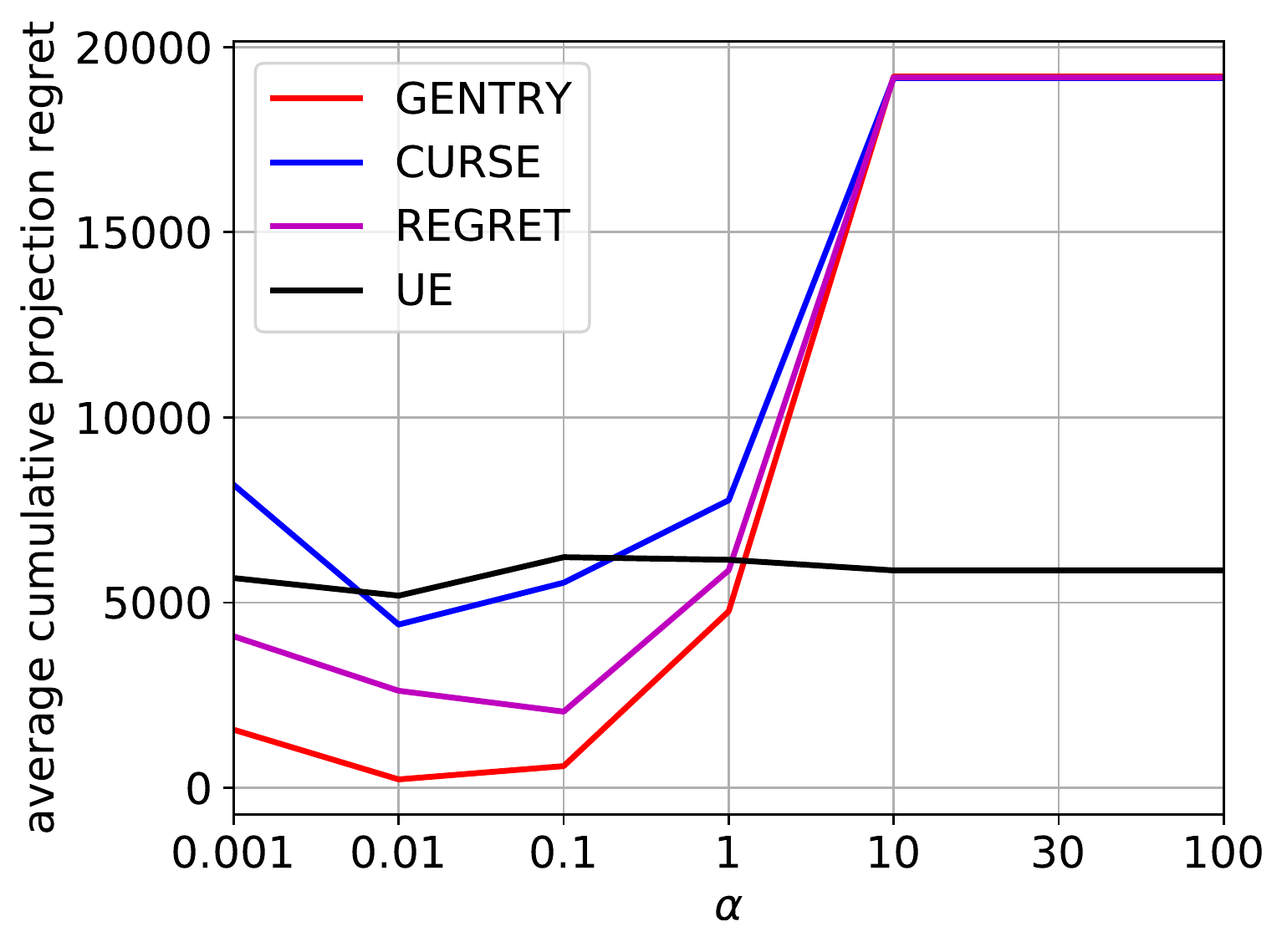}
        \caption{Using setting \ref{it:settingc} with  $d=4$, and $u = 2$}
        \label{fig:parameter_inf}
    \end{subfigure}
    \caption{Regret comparison between different strategies with $\vartheta=0.5$ and varying $\alpha$.}
\end{figure}

All the strategies depend on the hyperparameter $\alpha$, with a larger $\alpha$ corresponding to more exploration on average. We do simulations using setting \ref{it:settinga} with $d=10$, $K=45$, $\vartheta = 0.5$, $u = 5$, and setting \ref{it:settingc} with $d=4$, $\vartheta = 0.5$, $u = 2$. \cref{fig:parameter,fig:parameter_inf} show how the average cumulative projective regret at time step $10^4$ changes with different $\alpha$ in each strategy. We observe the following:
\begin{itemize}
  \item We note that a moderate $\alpha$ is optimal for each strategy. When $\alpha$ is too large, the strategy explores too frequently, leading to a large projection regret as expected. On the other hand, a small $\alpha$ results in little exploration, and good arms cannot be found efficiently. 
  \item In \cref{fig:parameter}, GENTRY with $\alpha = 1$ outperforms all the other benchmark strategies. This is because the other strategies do not achieve an asymptotically unbiased estimation of the projection reward. \cref{fig:parameter_inf} shows similar results. 

\end{itemize}
In the following simulations, for a fair comparison, we tune the parameter $\alpha$ for each strategy to achieve the best average asymptotic cumulative projection regret for that strategy. Specifically, we set $\alpha=1$ for all the strategies except UE, which is given a $\alpha=0.1$ when using setting \ref{it:settinga} and \ref{it:settingb}. When using setting \ref{it:settingc}, we set $\alpha=0.01$ for all the strategies except REGRET, which is given a $\alpha=0.1$.

\subsubsection{Results and Analysis}

\begin{figure}[!htb]
    \centering
    \begin{subfigure}[b]{0.9\columnwidth}
        \centering
	\includegraphics[width=\linewidth]{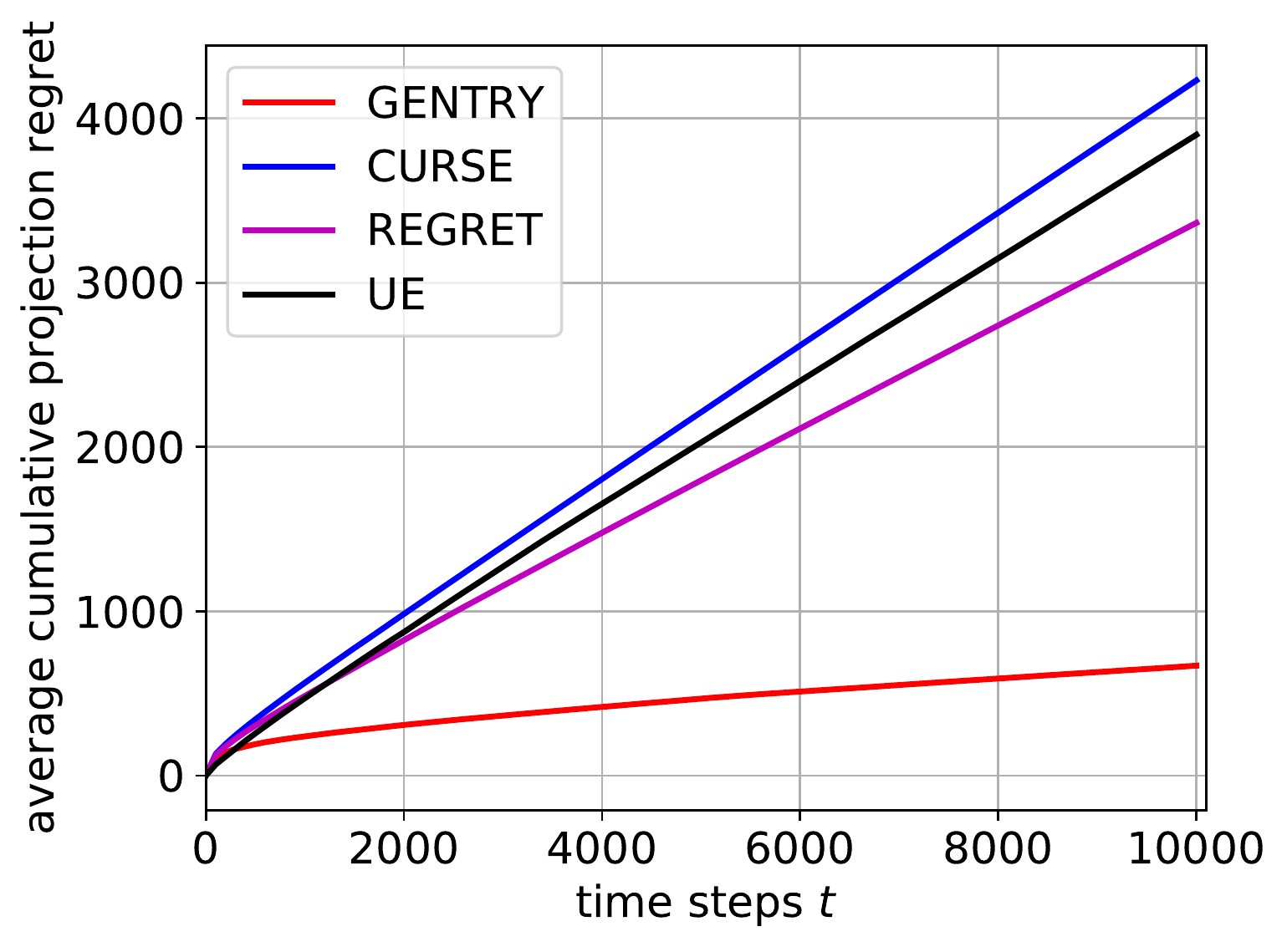}
        \caption{Average cumulative projection regret}
        \label{fig:regreta}
    \end{subfigure}\\
    \begin{subfigure}[b]{0.9\columnwidth}
        \centering
	\includegraphics[width=\linewidth]{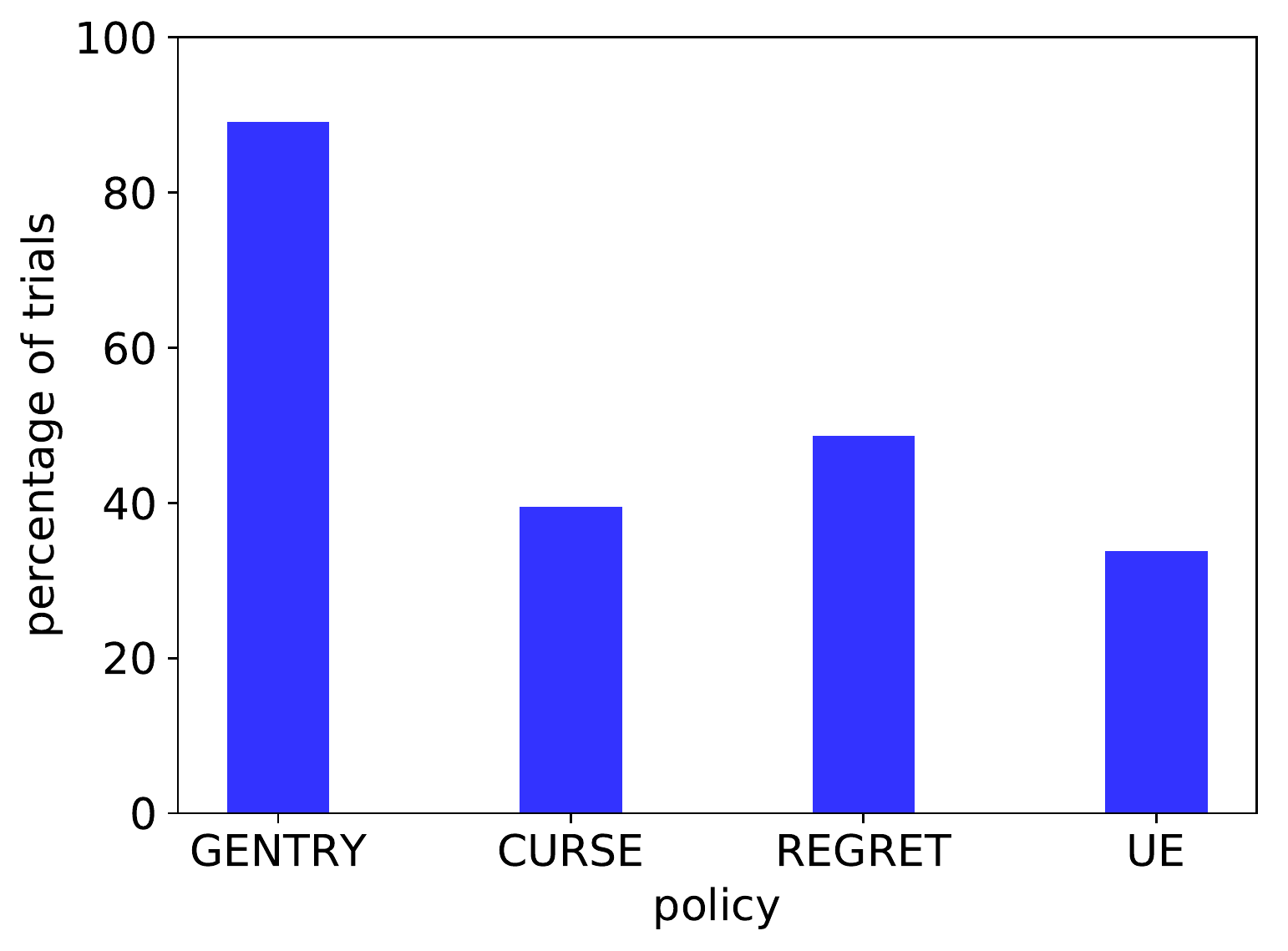}
        \caption{Percentage of trials where the best decision is found}
        \label{fig:n_besta}
    \end{subfigure}
    \caption{Performance comparison between different strategies using setting~\ref{it:settinga} with $d=10$, $K=45$, $\vartheta=0.5$ and $u = 5$. }\label{fig:result_a}
\end{figure}
\begin{figure}[!htb]
    \centering
    \begin{subfigure}[b]{0.9\columnwidth}
        \centering
	\includegraphics[width=\linewidth]{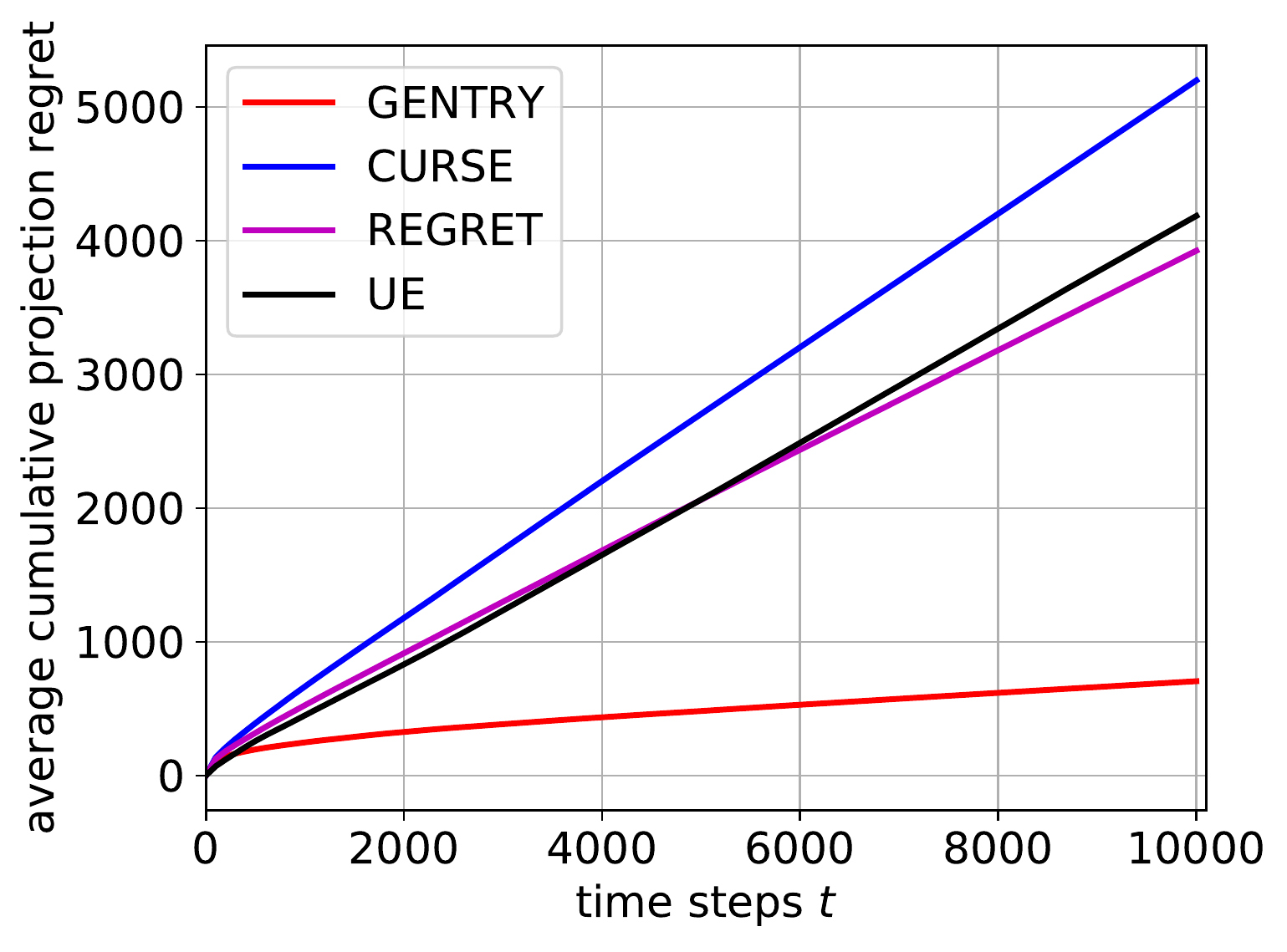}
        \caption{Average cumulative projection regret}
        \label{fig:regretb}
    \end{subfigure}\\
    \begin{subfigure}[b]{0.9\columnwidth}
        \centering
	\includegraphics[width=\linewidth]{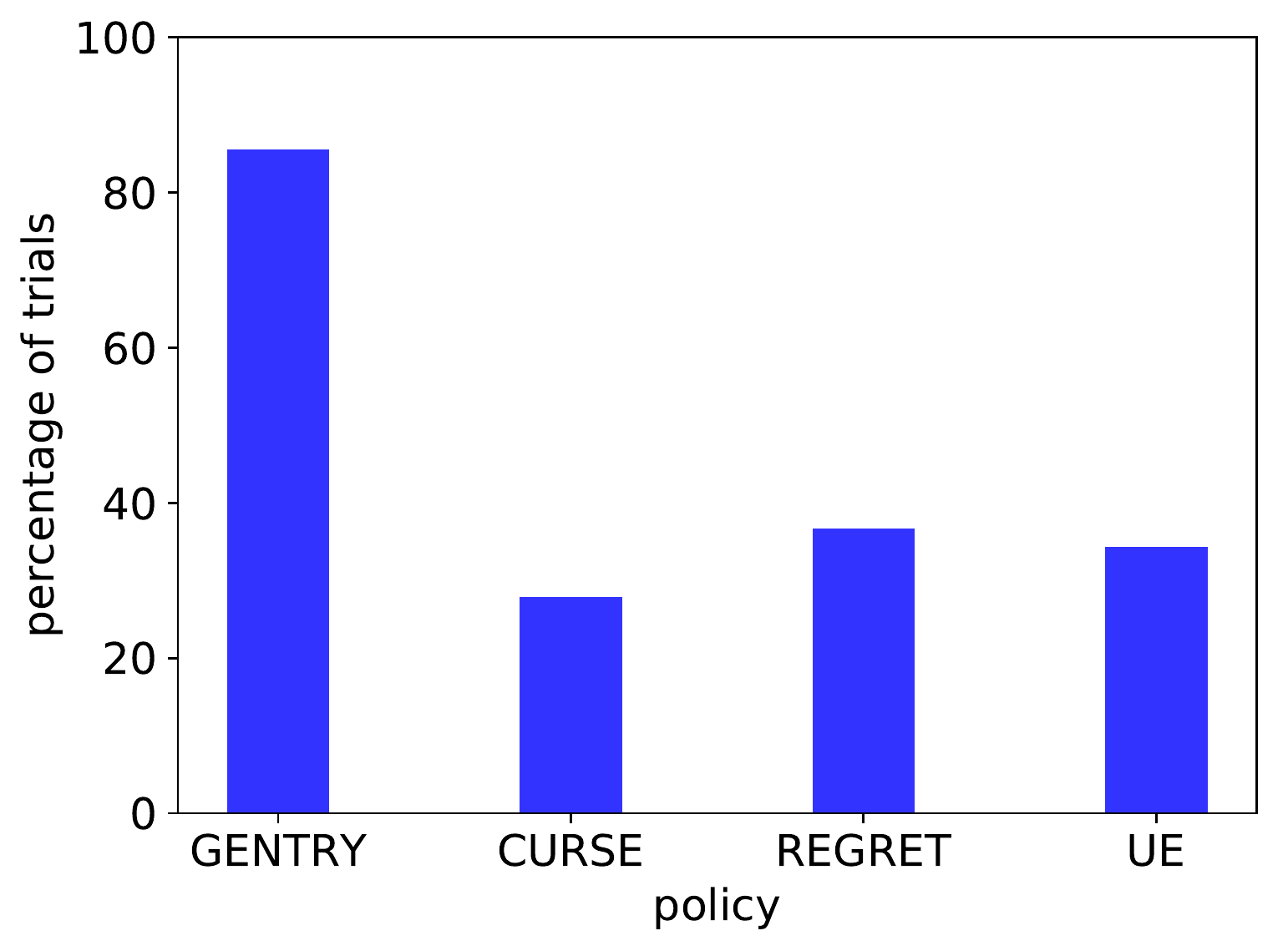}
        \caption{Percentage of trials where the best decision is found}
        \label{fig:n_bestb}
    \end{subfigure}
    \caption{Performance comparison between different strategies using setting~\ref{it:settingb} with $d=10$, $K=45$, $\vartheta=0.5$ and $u = 5$. }\label{fig:result_b}
\end{figure}

In the simulations for the finite-arm case using settings~\ref{it:settinga} and \ref{it:settingb}, we set $d=10$, $K=45$, $\vartheta=0.5$ and $u = 5$. The simulation results are shown in \cref{fig:result_a,fig:result_b}. We observe the following:
\begin{itemize}
  \item The average cumulative projection regrets of different strategies are shown in \cref{fig:regreta,fig:regretb}.  We see that GENTRY has obvious sublinear cumulative projection regret performance. The other benchmark strategies all suffer from a linearly increasing cumulative projection regret. This verifies that if our objective is to maximize the cumulative projection reward instead of the cumulative return (including the corruption), GENTRY is more appropriate. 
  \item \cref{fig:n_besta,fig:n_bestb} show the percentage of trials in which the optimal arm is chosen for more than $90\%$ of the time in the last $200$ time steps. We see that GENTRY finds the optimal arm in most trials, and outperforms the other strategies by a significant margin.
\end{itemize}

In the simulations for the infinite-arm case using setting~\ref{it:settingc}, we set $d=4$, $\vartheta=0.5$ and $u = 2$. From \cref{fig:regretc}, we also observe that GENTRY has obvious sublinear cumulative projection regret performance. The other benchmark strategies have linearly increasing cumulative projection regret. This verifies the efficiency of GENTRY for the infinite-arm case.

\begin{figure}[!htpb]
        \centering
	\includegraphics[width=\linewidth]{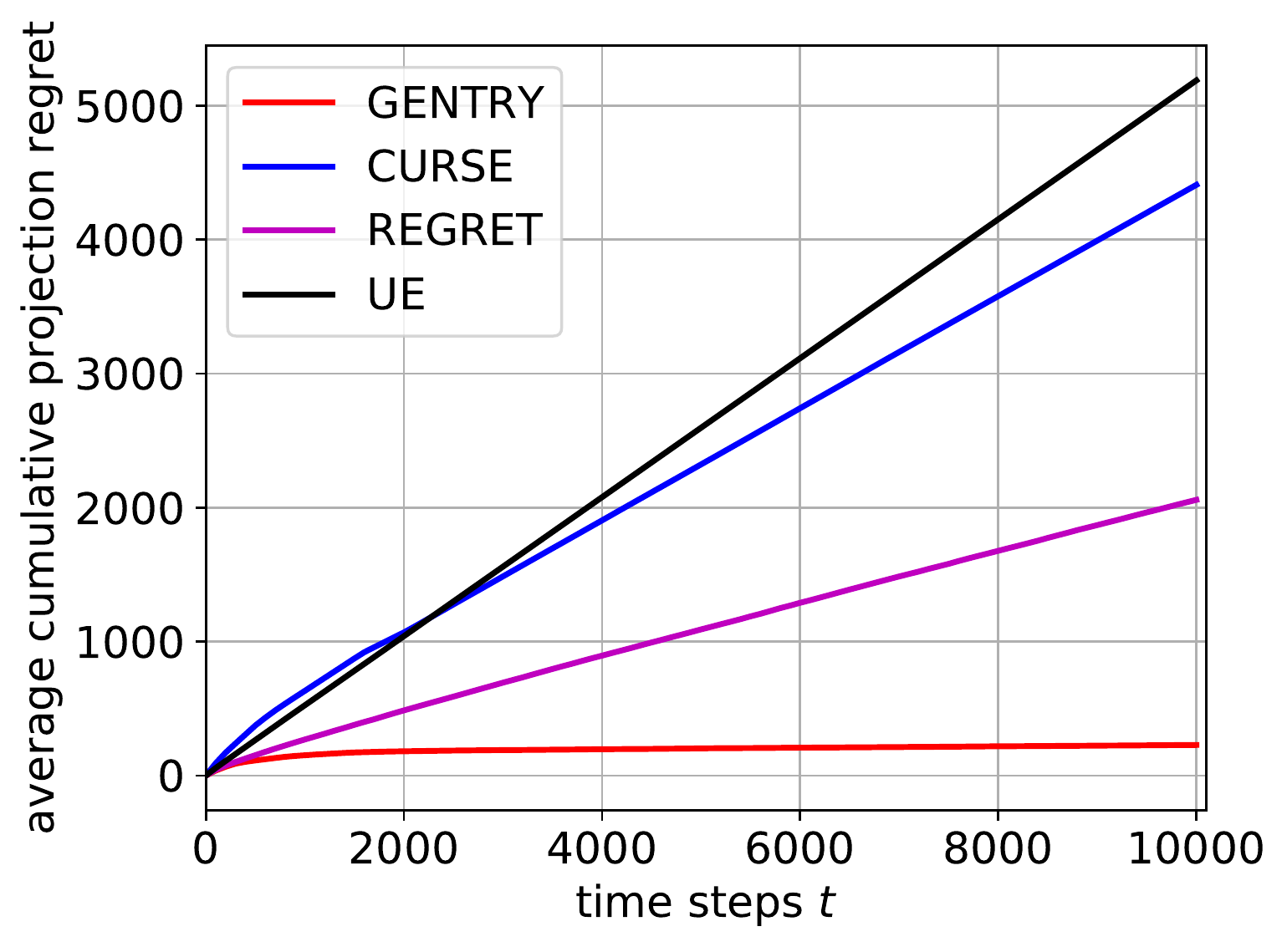}
        \centering	
	\caption{Average cumulative projection regret using setting \ref{it:settingc} with $d=4$, $u = 2$ and $\vartheta=0.5$.}       
    \label{fig:regretc}
\end{figure}

\subsection{Experiments on Wine Quality Dataset}\label{ssec:real_data}
We next compare the performance of different strategies using the wine quality dataset, which contains $11$-dimension description vectors of $4898$ white wines and their ratings (scores between $0$ and $10$). In each trial, we randomly select $200$ wines with ratings $4$, $5$, $6$, $7$ and $8$, as the decision set $\bbD$, since among all the wines there are only $5$ wines with ratings larger than $8$ and $20$ wines with ratings less than $4$. Each dimension of the wine description vector is a physicochemical characteristic like volatile acidity, chlorides, or density. Due to privacy and logistic issues\cite{cortez2009modeling}, only the physicochemical characteristics and ratings are available (e.g., there is no data about grape types, wine brand, wine selling price, etc.). We add one additional feature drawn \gls{iid} from $\calU(0,1)$ as the protected feature. The corresponding rating of each wine is then corrupted by subtracting $4$ times this protected feature value from the original rating. In this experiment, we take the original rating as the projection reward. Finally, we add a constant $1$ as the constant feature to each description vector. If we put the protected feature as the final dimension of each wine, then $d=13$, and the projection matrix $P_\bbU$ is defined as a diagonal matrix whose $(i,i)$ entry is 1 for $i=1,\cdots,12$, and $0$ for $i=13$. As there are finitel many arms, we use $\epsilon = \epsilon_t^f$ in this experiment. The results are shown in \cref{fig:regret_wine}, from which we observe that GENTRY outperforms all the other strategies.

\begin{figure}[!htpb]
  \centering
  \includegraphics[width=\linewidth]{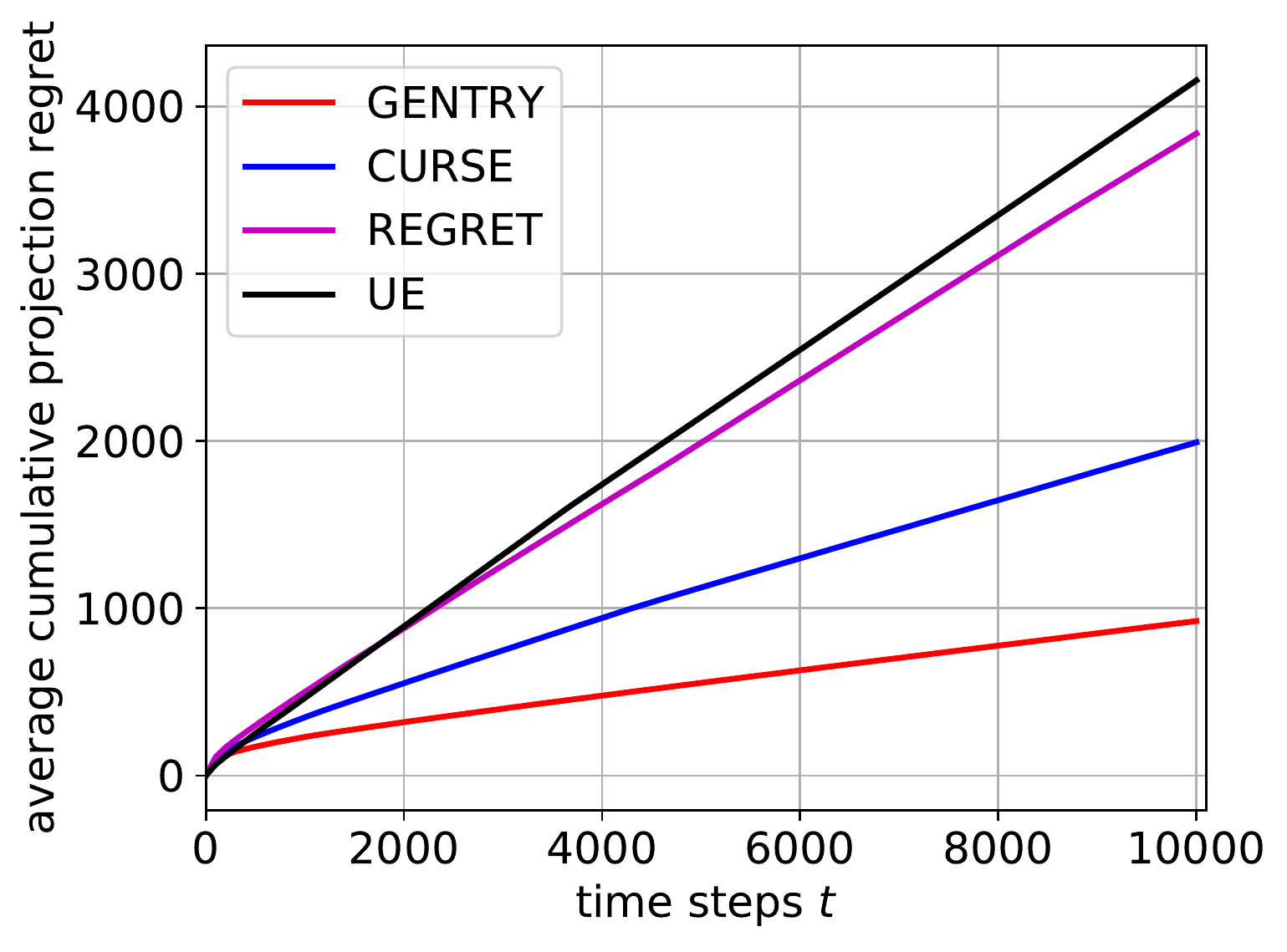}
 \centering
 \caption{Average cumulative projection regret using the wine quality dataset.}
		\label{fig:regret_wine}
\end{figure}

\section{Conclusion} \label{sec:conclusion}

We have formulated the orthogonal projection problem in the linear stochastic bandit model, where the objective is to maximize the cumulative projection reward over a subspace of decision attributes based on observed returns that consist of the projection reward with corruption. Our proposed GENTRY achieves sublinear projection regret for the finite- and infinite-arm cases. Experiments verify the efficiency of our strategy. Our formulation and strategy are useful in avoiding discrimination in recommendation systems and in mixed drug treatment trials.

In this paper, we have assumed that the target subspace is known beforehand when decomposing the reward. However, in practice, we may not know what is a suitable projection subspace \emph{a priori}. It is of interest in future research to develop methods to learn this subspace in conjunction with the best arm using possibly additional side-information.

\bibliographystyle{IEEEtran}
\bibliography{IEEEabrv,citeproj}

\end{document}